\documentclass{article}

\usepackage[preprint]{neurips_2025}

\usepackage[utf8]{inputenc} 
\usepackage[T1]{fontenc}    
\usepackage{hyperref}       
\usepackage{url}            
\usepackage{booktabs}       
\usepackage{amsfonts}       
\usepackage{nicefrac}       
\usepackage{microtype}      
\usepackage{xcolor}         
\usepackage{amsmath,amstext,amsfonts,bm,amssymb,mathtools,amsthm}
\usepackage{tcolorbox}
\usepackage{color,xcolor,xspace}
\usepackage{booktabs}
\usepackage{thm-restate}
\usepackage{bbding}
\usepackage{pifont}
\usepackage{wasysym}
\usepackage{wrapfig}

\usepackage[inline]{enumitem}
\theoremstyle{plain}
\newtheorem{theorem}{Theorem}[section]
\newtheorem{proposition}[theorem]{Proposition}

\theoremstyle{definition}

\theoremstyle{remark}

\usepackage{booktabs}
\usepackage{colortbl}
\usepackage{placeins}
\definecolor{mygray}{gray}{.9}

\newcommand\red{\textcolor{black}}

\usepackage[capitalize]{cleveref}
\usepackage{natbib}

\usepackage{caption}

\usepackage[utf8]{inputenc} 
\usepackage[T1]{fontenc}    
\usepackage{hyperref}       
\usepackage{url}            
\usepackage{booktabs}       
\usepackage{amsfonts}       
\usepackage{nicefrac}       
\usepackage{microtype}      
\usepackage{xcolor}         

\usepackage{placeins}

\title{Exploiting Layer Normalization Fine-tuning in Visual Transformer Foundation Models for Classification}

\author{
Zhaorui Tan$^{1,2,3\dag}$, 
Tan Pan$^{1,4\dag}$, 
Kaizhu Huang$^{5}$,
Weimiao Yu$^{6}$, \\ 
\textbf{Kai Yao$^{7}$,
Chen Jiang$^{1,4}$,  
Qiufeng Wang$^{2}$, 
Anh Nguyen$^{3}$,
Xin Guo$^{1,4}$,} \\   
\textbf{Yuan Cheng$^{1,4*}$,
Xi Yang$^{2}$}\thanks{Corresponding authors.
$^\dag$Equal contributions.
This research was conducted during an internship at the Shanghai Academy of Artificial Intelligence for Science.}
\\
$^{1}$Shanghai Academy of Artificial Intelligence for Science, $^{2}$Xi'an Jiaotong-Liverpool University
\\$^{3}$University of Liverpool, $^{4}$AI$^3$ Fudan University,
$^{5}$Duke Kunshan University, $^{6}$BII, A$^*$STAR
\\
 $^{7}$ Zhejiang University
{\tt\small  raytan@liverpool.ac.uk, Xi.Yang01@xjtlu.edu.cn} \\
{\tt\small  chengyuan@sais.com.cn}
}

\begin{document}

\maketitle

\begin{abstract}
LayerNorm is pivotal in Vision Transformers (ViTs), yet its fine-tuning dynamics under data scarcity and domain shifts remain underexplored. 
This paper shows that shifts in LayerNorm parameters after fine-tuning (LayerNorm shifts) are indicative of the transitions between source and target domains; its efficacy is contingent upon the degree to which the target training samples accurately represent the target domain, as quantified by our proposed Fine-tuning Shift Ratio ($FSR$).
Building on this, we propose a simple yet effective rescaling mechanism using a scalar $\lambda$ that is negatively correlated to $FSR$ to align learned LayerNorm shifts with those ideal shifts achieved under fully representative data, combined with a cyclic framework that further enhances the LayerNorm fine-tuning.
Extensive experiments across natural and pathological images, in both in-distribution (ID) and out-of-distribution (OOD) settings, and various target training sample regimes validate our framework. 
Notably, OOD tasks tend to yield lower $FSR$ and higher $\lambda$ in comparison to ID cases, especially with scarce data, indicating under-represented target training samples. Moreover, ViTFs fine-tuned on pathological data behave more like ID settings, favoring conservative LayerNorm updates. 
Our findings illuminate the underexplored dynamics of LayerNorm in transfer learning and provide practical strategies for LayerNorm fine-tuning.
\end{abstract}

\section{Introduction}


Vision Transformers (ViTs)~\cite{dosovitskiy2020image} and their variants dominate visual foundation models (ViTFs), incorporating Layer Normalization (LayerNorm) typically before or after attention layers~\cite{radford2021learning,oquab2023dinov2,lu2024visual,wang2024pathology,ding2024multimodal}. Although ViTFs efficiently transfer across numerous downstream tasks and serve as pretrained backbones, full-model fine-tuning from \red{source to target domains} is resource-intensive and often inefficient~\cite{jie2023fact,de2023effectiveness,zhao2023tuning}. To address this, fine-tuning only the LayerNorm parameters (LayerNorm fine-tuning)~\cite{sandler2023parameter} offers an alternative.
Recent studies~\cite{zhao2023tuning,chen2024efficiency,de2023effectiveness,qi2022parameter,giannou2023expressive} show that fine-tuning only LayerNorm parameters offers an efficient approach, yielding competitive or superior results with minimal model weight updates. In particular, \cite{giannou2023expressive} investigates the expressive power of normalization layers, demonstrating that for random ReLU networks, fine-tuning just the normalization layers can reconstruct any target network that is $O\sqrt{width}$ times smaller.
Although LayerNorm fine-tuning is effective, its mechanisms in ViTFs is not well understood.

This paper systematically investigates LayerNorm's role in parameter-efficient fine-tuning, highlighting its behavior across domain shifts and varying data regimes - critical factors for real-world applications. 
We aim to address the subsequent inquiries:
(i) How does the domain shift between the source and target domains impact LayerNorm statistics and parameters? 
(ii) How do LayerNorm parameter shifts after fine-tuning (LayerNorm shifts) affect performance across various data conditions?
(iii) Is it possible to further improve the fine-tuning of LayerNorm in practical applications?

\begin{wrapfigure}{r}{0.45\linewidth}
  \centering
  \includegraphics[width=0.95\linewidth]{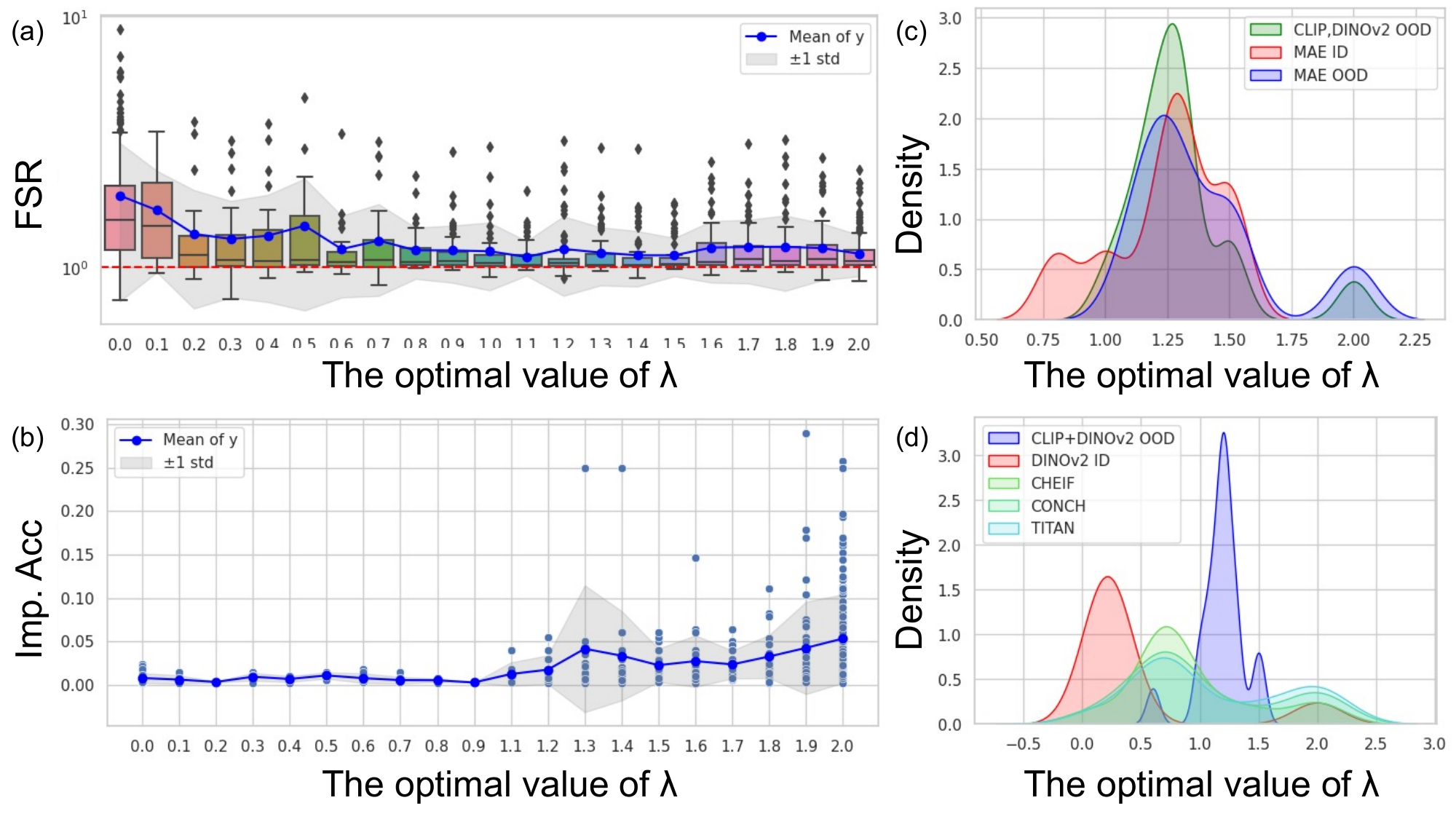}
  \vspace{-0.2cm}
    \caption{ Toy example results: (a) $\lambda$ scales proportionally with FSR; (b) Accuracy improvements versus optimal $\lambda$; 
    (c, d) Kernel density estimation visualization of the optimal (best) values of  $\lambda$  (c) of natural image ViTFs (from \cref{tab:ID_OOD_res}) across different settings for the DomainNet dataset and
    (d) of pathological ViTFs (CHEIF, CONCH, and TITAN) compared with natural images ViTFs (CLIP and DINOv2) (from \cref{tab:ID_OOD_res,tab:natural_image_res,fig:path_res_details}).
    }
  \label{fig:banner}
  \vspace{-0.5cm}
\end{wrapfigure}
To thoroughly study the above issues, we discuss diverse fine-tuning scenarios: in-distribution (ID) versus out-of-distribution (OOD) fine-tuning; one-shot, few-shot, and standard fine-tuning regimes; and natural versus pathological image domains.
Our analysis reveals key findings for the discussed questions:

\textbf{(I)} LayerNorm shifts effectively encode distributional divergence between source and target data, regardless of ID or OOD tuning settings (\cref{sec:Tuning LayerNorm layers in the model collectively capture the data distribution shifts}).

\textbf{(II)} 
LayerNorm fine-tuning performance relies on the target training data's representation of the target domain under source domain influence.
To quantify this, we introduce the Fine-tuning Shift Ratio ($FSR$) as the ratio of the labeled target training set shift to the full target domain shift, both measured against the source domain. The ratio of learnt to ideal LayerNorm shifts is proportional to $FSR$. When $FSR=1$, learnt LayerNorm shifts approximate ideal LayerNorm shifts (\cref{sec:Tuning LayerNorm layers in the model collectively capture the data distribution shifts}).

\textbf{(III)} Rescaling partial LayerNorm shifts via a non-negative single-valued scalar $\lambda$ (\cref{fig:banner}~(b)) can mirror the scenario where $FSR=1$ and then improve the performance. 
The optimal value of $\lambda$ is tightly linked to $FSR$ (\cref{fig:banner}~(a)), emphasizing $\lambda$'s role as an indicator for the actual $FSR$ without the source domain.
This insight is supported by a series of controlled toy experiments (\cref{sec:Rescaling LayerNorm shifts leads to improvements}). 
Building on these findings, we propose a robust tuning framework that alternates between training the predictor and fine-tuning LayerNorm layers, enabling them to better capture data shifts and consistently improve performance across different training sample regimes (\cref{sec:method}).

During the experiments, we have two further empirical findings in
\cref{{sec:exp}}:
\textbf{(IV)} 
Optimal $\lambda$ values inferred by $FSR$ tend to be less than $1$ in ID but greater than $1$ in OOD settings, particularly under the same number of limited target training samples. This suggests that the $FSR$ in OOD scenarios is generally lower, indicating less representativeness of the target training data relative to the full target distribution  (\cref{fig:banner}~(c)). 
\textbf{(V)} Most ViTFs fine-tuned on pathological image datasets exhibit behavior aligned with ID settings and tend to prefer $\lambda \le 1$. This reveals the critical differences between pathological and natural image tasks, whereas pathological images may often exhibit less variation than natural images (\cref{fig:banner}~(d)). 
See related work in \cref{sec:Related work}. The codebase, data splits, and resources will be publicly released upon acceptance to enable reproducibility and further research.

\textbf{Preliminaries.}
We recall the definition of standard LayerNorm for readers as follows:
\begin{equation}
\label{eq:ln}
    {ln}(*)=\left ({(z-\mathrm{E}[z])}/{ \sqrt{\operatorname{Var}[z]+\epsilon} }\right ) \odot \gamma+\beta.
\end{equation}
where $z$ is the input variable, $\gamma,\beta$ are the learnable scale and bias,
and $\epsilon$ is a small value to prevent numerical problems.
Specifically, after the LayerNorm operation, $ln(z) \sim \mathbb{P}_{\beta, \gamma^2}$, where $\beta$ is the variance mean and $\gamma^2$ is the variance; $\mathbb{P}_{\beta, \gamma^2}$ is a distribution whose form is not explicit. 

Standard ViTFs with $n$ attention blocks use standard LayerNorm; The $n^{th}$ block defines attention and other layers, where $ln^{pre}$ and $ln^{post}$ specify pre- and post-attention LayerNorm, as shown in {$\mathcal{M} = \{ enc, [m'_1, ln^{pre}_{1}, m_1, ln^{post}_{1}],  [m'_2, ln^{pre}_{2}, m_2, ln^{post}_{2}],\dots,[m'_n, ln^{pre}_{n}, m_n, ln^{post}_{n}], \}$}. 
Superscripts $S$, $T$, and $T*$ indicate variables from the source, target, and target training domains. Source data and labels are $(X^{S}, Y^{S})$, target training data and labels are $(X^{T}, Y^{T})$, and all target data are $(X^{T*}, Y^{T*})$. The original pretrained, tuned, and ideal tuned models correspond to $\mathcal{M}^{S}$, $\mathcal{M}^{T}$, and $\mathcal{M}^{T*}$, respectively.
Normally, additional layers $\mathcal{C}$ as the predictor are plugged at the end of $\mathcal{M}$ for specific tasks such as classification. To avoid unexpected effects from $\mathcal{C}$, this paper mainly uses lossless linear layers for $\mathcal{C}$ and conducts classification tasks. 
For simplification, we additionally denote all the LayerNorm layers $\mathcal{LN}$ in the $\mathcal{M}^{S}$ $\mathcal{M}^{T}$ and  $\mathcal{M}^{T*}$ as $\mathcal{LN^{S}}$ $\mathcal{LN}^{T}$ and  $\mathcal{LN^{T*}}$.

\textbf{Assumptions.} 1. For any input $X^{S}$ from self-supervised ViTFs, there exists a corresponding $Y^{S}$ that aligns with the supervised task  $(X^{T}, Y^{T})$;  2. The ViTFs can extract useful (though potentially suboptimal) features from  $X$ without fine-tuning; 3. A predictor $\mathcal{C}$  exists that can map these features to $Y$.
These naturally hold for most large-scale pre-trained vision transformers.

\section{Relationship between LayerNorm shifts and domain shifts}
\label{sec:Tuning LayerNorm layers in the model collectively capture the data distribution shifts}

\textbf{LayerNorm shifts encode domain shifts.}
This part shows that all LayerNorm layers in the model collectively capture the data distributions, while the other layers, except LayerNorm layers, are fixed. 
\begin{proposition}
\label{prop:shift_compare}
Under the assumption that the ${Shift}_{ln}, {Shift}_{data}$ functions are proportionally comparable across domains. This implies that the magnitude of change in LayerNorm parameters tracks the magnitude of data distribution shift. While this may not be strictly true in all cases, it serves as a reasonable approximation for theoretical analysis and is empirically supported.
The shift between the original to fine-tuned LayerNorm layers is proportional to the shift between $X^{S}$, $X^{T}$.
\begin{equation}
\label{eq:shift_compare}
     |{Shift}_{ln}(\mathcal{LN}^{S}, \mathcal{LN}^{T})| \propto |{Shift}_{data}( (X^{S}, Y^{S} ) , (X^{T}, Y^{T}))|,
\end{equation}
where ${Shift}(\cdot)$ is an arbitrary quantification that measures the shifts between two variables. 
\end{proposition} 

\begin{proof}
\cref{eq:shift_compare} follows by linking $\mathcal{LN}^{S}$ and $\mathcal{LN}^{T}$ through loss function definitions and assumptions.
Details are in \cref{app:math_details}.
\end{proof}

\red{
The overall data distribution shift between the source and target domains can be decomposed into input and label components:
$|{Shift}_{data}((X^{S}, Y^{S}), (X^{T}, Y^{T}))| \Rightarrow |{Shift}_{data}(X^{S}, X^{T})| + |{Shift}_{data}(Y^{S}, Y^{T})|.$
Here, $|{Shift}_{data}(Y^{S}, Y^{T})|$ refers to the label shift, which is typically negligible in many real-world scenarios, as label distributions tend to be relatively stable across domains for classification.  
Therefore, we treat it as a constant in our analysis and focus primarily on the  $|{Shift}_{data}(X^{S}, X^{T})|$, which is more likely to vary significantly and affect model adaptation.
Thus, ${Shift}_{data}( (X^{S}, Y^{S} ) , (X^{T}, Y^{T})) \Rightarrow {Shift}_{data}( X^{S} , X^{T})$.} 

While other types of layers beyond LayerNorm can also capture distributional shifts, the representations they learn are often entangled with other factors such as semantic content. In contrast, LayerNorm primarily captures distributional variations in a more disentangled manner, focusing on normalization statistics without encoding additional task-specific knowledge. This property makes LayerNorm particularly suitable for isolating and adapting to domain shifts.

For ${Shift}_{ln}\left(\mathcal{LN^T}, \mathcal{LN^S}\right)$, 
it should reflect the global shifts of all LayerNorm layers: ${Shift}_{ln}(\mathcal{LN}^T, \mathcal{LN}^S):= \| \mathcal{LN}^T - \mathcal{LN}^S \|_{\mathcal{G}}$. Then, we can obtain:
\begin{equation}
\begin{split}
\label{eq:ln_shift}
    \| \mathcal{LN}^T - \mathcal{LN}^S \|_{\mathcal{G}} 
    := \sum\nolimits_{i=1}^n \left( \left \| \gamma_i^T - \gamma_i^S \right \|_2 + \left \| \beta_i^T - \beta_i^S \right \|_2 \right). 
\end{split}
\end{equation}
Meanwhile, 
${Shift}(X^{S}, X^{T})$ can be various forms. Here, we adopt the Wasserstein distance to quantify ${Shift}(X^S, X^T)$ as it captures the overall geometric discrepancy between feature distributions rather than merely comparing their mean or variance:
\begin{equation}
    {Shift}_{data}( X^{S} , X^{T}) := Wasserstein(X^{S}; X^{T}).
\end{equation}
This is particularly appropriate given that LayerNorm layers are sensitive to both mean and variance statistics, while the Wasserstein metric aligns naturally with its normalization objectives.

\textbf{Learned LayerNorm shifts versus ideal LayerNorm shifts.}
We now study the gap between $\mathcal{M}^{S}$, $\mathcal{M}^{T}$, and $\mathcal{M}^{T*}$.
As shown in \cref{eq:shift_compare}, which links the shifts between LayerNorm and data distributions, if the distribution shifts between $X^{S}$, $X^{T}$, and $X^{T*}$ are characterized, the corresponding LayerNorm shifts can also be clarified.
Here, we define $X^T$ as the finite set of target-domain training samples used for fine-tuning, while $X^{T*}$ denotes the full target-domain data distribution (e.g., all available test samples).
Since $X^T$ is typically a limited subset of $X^{T*}$, it may not fully capture the statistical properties of the full target domain. Thus, the distribution shift $shift_{data}(X^S, X^T)$ can differ from $shift_{data}(X^S, X^{T*})$.
Consider the $shift_{data}(X^{S}, X^{T})$ and $shift_{data}(X^{S}, X^{T*})$, the Fine-tuning Shift Ratio ($FSR$) is defined as 
\begin{equation}
    FSR:={Shift_{data}(X^{S}, X^{T})}/({Shift_{data}(X^{S}, X^{T*}) + \epsilon}),
\end{equation}
where the $\epsilon$ is a small value to avoid numeric issues. Note that the shift is calculated in the raw space, and $X^{T}$ and $X^{T*}$ are from the same distribution, but. 
Following \cref{eq:shift_compare}:
\begin{equation}
\label{eq:shift_compare_2}
     \frac{{Shift}_{ln}(\mathcal{LN}^T, \mathcal{LN}^S)}{{Shift}_{ln}(\mathcal{LN}^{T*}, \mathcal{LN}^S)}  \propto FSR \Rightarrow  \frac{{Shift}_{ln}(\mathcal{LN}^T, \mathcal{LN}^S)}{{Shift}_{ln}(\mathcal{LN}^{T*}, \mathcal{LN}^S)}  = f(FSR),
\end{equation}
where $f(\cdot)$ is a monotonic mapping function that captures the proportional relationship between $FSR$ and the distributional shifts. 
To prevent numeric issues, a small value like $1e-8$ will be added to the denominator of all subsequent fraction terms, which are omitted.
From \cref{eq:shift_compare_2}, LayerNorm fine-tuning performance is influenced by the relative shift between target training samples and the source, rather than the absolute shift, compared to the full target domain and the source. In other words, \textit{the effectiveness of LayerNorm fine-tuning does not inherently depend on whether the target data is in- (ID) or out-of-distribution (OOD) relative to the source.} Rather, it depends on how well the target training data represents the overall target domain's distribution, especially regarding the source-target divergence.See \cref{sec:ID LayerNorm fine-tuning vs OOD LayerNorm fine-tuning} for empirical analysis.


\section{Rescaling LayerNorm shifts to mimic FRS=1 scenario}
\label{sec:Rescaling LayerNorm shifts leads to improvements}

Diving into \cref{eq:shift_compare_2}, we discuss following scenarios:

\textbf{Scenario one: $FSR=1$.} $FSR=1$ indicates the scenario where $X^{T}$ is statistical sufficient for representing $X^{T*}$. In such ideal cases, the fine-tuned LayerNorm layers can generalize for all $X^{T*}$.

\textbf{Scenario two: $FSR<1$ or $FSR > 1$.} While $FSR$ dose not equals to one, the $X^{T}$ is no longer statistical sufficient for representing $X^{T*}$. Specifically, while $FSR> 1$, we have:
\begin{equation}
\label{eq:why_scale}
\begin{split}
    \frac{{Shift}_{ln}(\mathcal{LN}^T, \mathcal{LN}^S)}{{Shift}_{ln}(\mathcal{LN}^{T*}, \mathcal{LN}^S)} > f(FSR=1)
    \Rightarrow \frac{\lambda \cdot {Shift}_{ln}(\mathcal{LN}^T, \mathcal{LN}^S)}{{Shift}_{ln}(\mathcal{LN}^{T*}, \mathcal{LN}^S)} = f(FSR=1),
\end{split}
\end{equation}
where $\lambda \ge 0$ and similar scenario for $FSR < 1$. Empirically, we notice that normally the $FSR \geq 1$ (see \cref{sec:Toy examples}). 
This indicates that the tuned LayerNorm layers can be further improved with rescaling. 
To investigate the rescale approach, we now consider only the $i^{th}$ layer of LayerNorm. Following \cref{eq:ln_shift}, the left-hand sized of \cref{eq:why_scale} can simplified as:
\begin{equation}
\begin{split}
    \frac{{Shift}_{ln}(\mathcal{LN}_i^T, \mathcal{LN}_i^S)}{{Shift}_{ln}(\mathcal{LN}_i^{T*}, \mathcal{LN}_i^S)} = \frac{ \left \| \gamma_i^T - \gamma_i^S \right \|_2 + \left \| \beta_i^T - \beta_i^S \right \|_2 }{ \left \| \gamma_i^T - \gamma_i^S \right \|_2 + \left \| \beta_i^{T} - \beta_i^S \right \|_2}.
\end{split}
\end{equation}
We observe that the convergence behaviors of the normalization parameters $\beta$ (mean) and $\gamma$ (variance) differ with respect to the number of target training samples $n$. 
Due to the Central Limit Theorem, the sample mean (i.e., proxy for $\beta$) converges rapidly as $n$ increases, with estimation error decreasing at a rate of $\mathcal{O}(1/\sqrt{n})$. 
In contrast, variance estimation (proxy for $\gamma$) relies on the chi-squared distribution and typically converges more slowly, requiring significantly more samples to achieve comparable confidence. 
This asymmetry suggests that fewer samples are needed to adapt $\beta$ reliably than $\gamma$, which is consistent with our empirical findings in Appendix~\cref{fig:natural_image_bias_weight}. 
For detailed derivations and sample complexity analysis of mean and variance estimation, please refer to \cref{app:math_details}.
As analyzed above, the $\mu$ convergences are quicker than $\sigma$ when given an increasing number of samples, implying similar phenomena between $\beta$ and $\gamma$. 
Moreover, the $\beta$ implies direction information while $\sigma$ does not. While $\beta$ is not converged, its direction may also not converge, and rescaling it would be risky.   
Therefore, the rescale should be applied only to the $\gamma^{T}$. Finally, the $\lambda$ is applied to all $\beta^{T}$ which achieves:
\begin{equation}
    \frac{\lambda \cdot {Shift}_{ln}(\mathcal{LN}^T, \mathcal{LN}^S)}{{Shift}_{ln}(\mathcal{LN}^{T*}, \mathcal{LN}^S)} : = \frac{ \sum_{i}^{n} \left(  \left \|  \lambda \cdot \gamma_i^T - \gamma_i^S \right \|_2 + \left \| \beta_i^T - \beta_i^S \right \|_2 \right)}{ \sum_{i}^{n} \left(\left \| \gamma_i^{T*} - \gamma_i^S \right \|_2 + \left \| \beta_i^{T*} - \beta_i^S \right \|_2 \right)} = f(FSR=1).
\end{equation}

\subsection{Toy examples: the relationship between $\lambda$, $FSR$ and performance}
\label{sec:Toy examples}

\begin{figure}[t]
  \centering
  \includegraphics[width=0.95\linewidth]{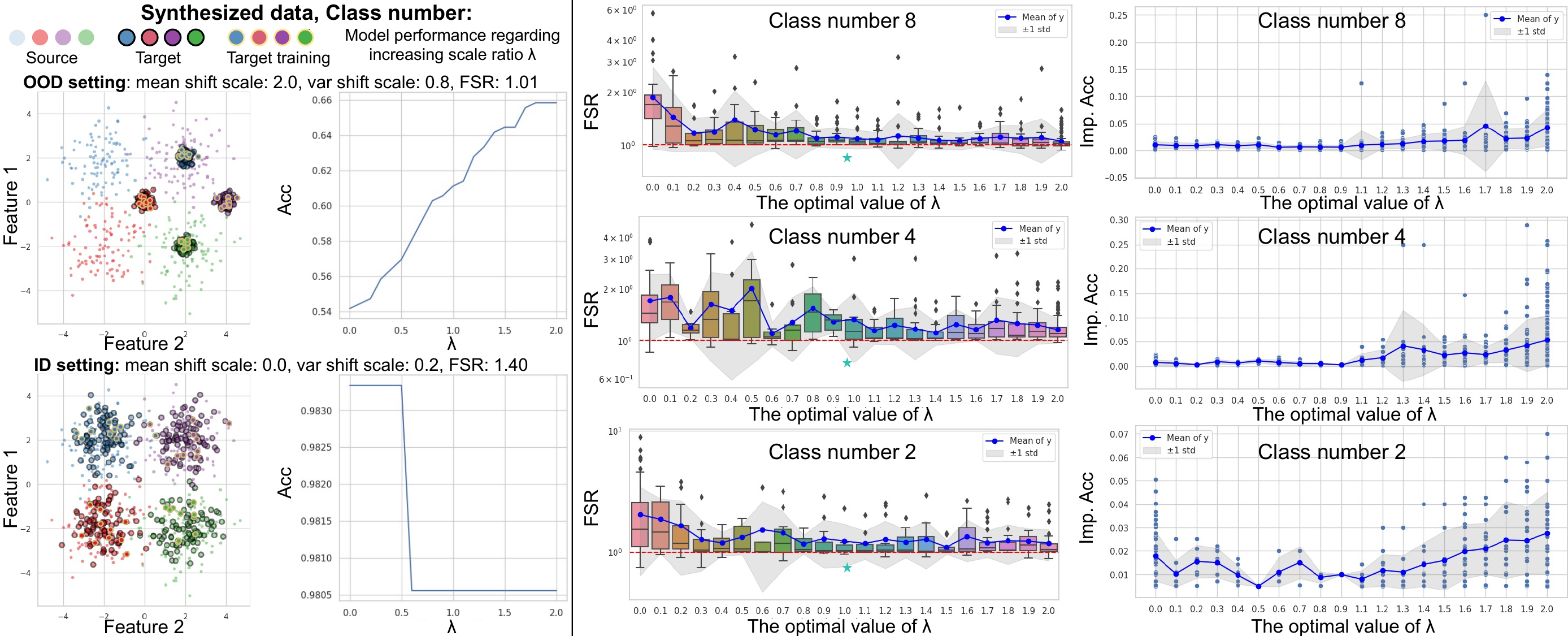}
    \caption{Toy example results. Left: Two generated data examples and their rescaling performance. The target training samples are $10\%$ of all target samples. Right: Statistical visualization of $FSR$ and the optimal scaling factor $\lambda$ across different class numbers. Note: Models with zero accuracy or unchanged accuracy after rescaling are excluded. See code and results in the supplementary materials.
    }
  \label{fig:toy_res}
  \vspace{-0.5cm}
\end{figure}

Computing the $FSR$ is often impractical during fine-tuning due to the large scale and limited access to pretraining datasets. Pathological foundation models \cite{lu2024visual,wang2024pathology,ding2024multimodal} typically use private datasets, complicating direct FSR measurement. To overcome this, we create controlled toy examples to empirically test the proposed $\lambda$ and investigate its link with $FSR$.

\textbf{Toy example training details and toy data generation.} 
We use a simple toy model (2-layer MLP with one LayerNorm) to simulate fine-tuning under domain shift. Synthetic data are class-conditional Gaussians with controlled mean and variance shifts to create source and target domains. We vary the number of labeled training target samples and shift levels to emulate different fine-tuning settings. After LayerNorm fine-tuning, we rescale $\gamma$ by $\lambda \in [0, 2]$ and identify the optimal $\lambda$ that maximizes accuracy. See \cref{app:Toy example details} for detailed descriptions of these toy examples.

\begin{wrapfigure}{r}{0.25\linewidth}
    \centering
    \captionof{table}{Overall statistics of toy example results. Not affected cases: In the cases where the accuracy does not change with respect to scaling ratio values.
    Avg. Imp.: Averaged improvements of all improved cases. 
    }
    \label{tab:toy_res_statistic}
    \resizebox{\linewidth}{!}{%
    \begin{tabular}{ll}
    \toprule
    Overall cases & {1815} \\
    Not affected cases & 290 \\
    Cases improved by $\lambda$ & 1326 \\
    \midrule
    Avg. Imp. &  0.0239 (2.39\%) \\
    \bottomrule
    \end{tabular}%
    }

\vspace{-0.5cm}
\end{wrapfigure}
\textbf{Rescaling $\gamma$ with an appropriate $\lambda$ enhances target-domain performance.}
As shown in \cref{tab:toy_res_statistic}, rescaling yields non-decreasing accuracy in over $90\%$ of cases and leads to measurable gains in $73\%$ of them, with an average improvement of $2.39\%$.
\textbf{The optimal value of $\lambda$ correlates negatively with $FSR$ yet positively with performance gains, suggesting that $\lambda$ can serve as an indicator for FSR when source data is inaccessible but target data is available.}
As shown in \cref{fig:banner}~(a), high $FSR$ values correspond to optimal $\lambda < 1$, while low $FSR$ values favor $\lambda > 1$. Moreover, \cref{fig:banner}~(b) highlights that $\lambda > 1$ configurations yield significantly higher gain, even exceeding $25\%$, compared to $\lambda < 1$ scenarios.

\textbf{The LayerNorm fine-tuning performance depends on how well the target training data represents the target domain under the influence of the source domain.}
\cref{fig:toy_res} left-hand side further illustrates this trend: in OOD settings with sufficient target data (lower $FSR$), $\lambda > 1$ leads to steady accuracy improvements, whereas in ID settings with limited data (larger $FSR$), $\lambda < 1$ results in smaller gains.
This further supports that the performance of LayerNorm fine-tuning ties to how representative the target training data is to the target domain, rather than directly ID or OOD settings.


\textbf{The observed tendencies are insensitive to the number of classes.}
As shown in the right-hand side of \cref{fig:toy_res}, the same relationships between $FSR$, $\lambda$, and fine-tuning improvements are consistently observed across different numbers of classes. This suggests that neither $FSR$ nor the optimal $\lambda$ is significantly affected by the label space $Y$. 

\textbf{Larger $FSR$ correlates with greater variance in the optimal value of $\lambda$, indicating increased randomness in LayerNorm convergence.}
As shown in \cref{fig:banner}~(a)(b) and \cref{fig:toy_res}, smaller values of $FSR$ correspond to reduced variance in the optimal $\lambda$. This suggests that while $\lambda < 1$ tends to be beneficial for larger $FSR$, it is generally less effective compared to $\lambda > 1$ in cases with smaller $FSR$, highlighting the variability in LayerNorm fine-tuning effectiveness while $FSR$ is large.

\section{Boosting LayerNorm fine-tuning for real-world application}
\label{sec:method}
\begin{wrapfigure}{r}{0.4\linewidth}
\vspace{-0.6cm}
  \centering
  \includegraphics[width=\linewidth]{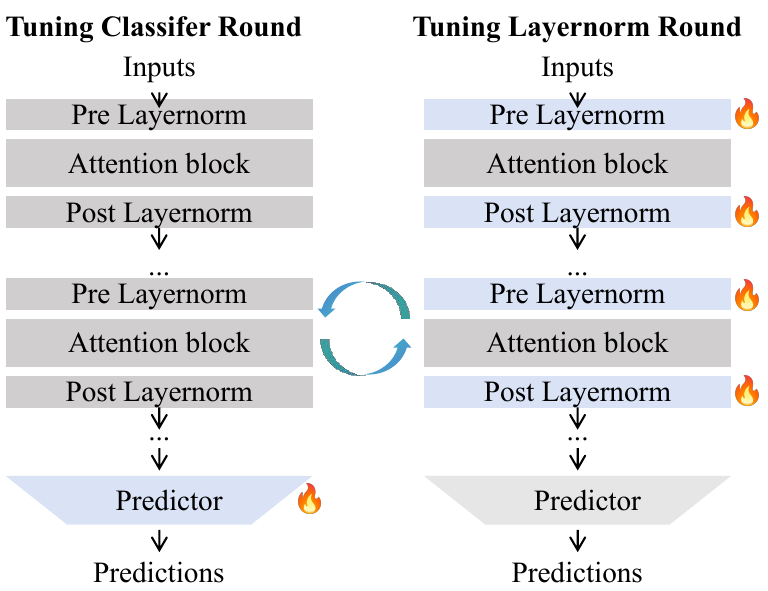}
    \caption{Diagrams of the proposed training approach. One complete turn contains both rounds. 
    }
  \label{fig:training}
    \centering
  \includegraphics[width=\linewidth]{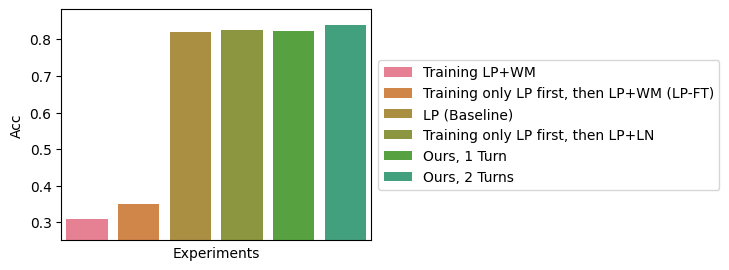}
  \vspace{-0.4cm}
    \caption{Comparisons of different fine-tuning strategies. Results show average performance across varying fine-tuning data fractions on the Bach dataset using the CONCH model. WM denotes the whole model $\mathcal{M}$; see experimental details in \cref{sec:path_exp} and additional results in \cref{app:More details about the proposed method}. 
    }
  \label{fig:path_diff_ft_ablation}
\vspace{-1.1cm}
\end{wrapfigure}
One presumption of when \cref{eq:shift_compare} holds is that other parameters, including the predictor, are fixed.
However, most ViTFs do not provide a $\mathcal{C}$ for every downstream task, where the $\mathcal{C}$ needs to be constructed during the fine-tuning stage. 
Otherwise, tuning LayerNorm layers with predictors (such as approaches proposed in \cite{zhao2023tuning}) may cause unstable capturing of the target data distributions and lead to performance degradation.

As \cref{fig:training} shows, our approach avoids such unstable optimization by training the predictors at first while the $\mathcal{M}$ is fixed (corresponding to \cref{eq:loss_diff} Data distribution gap term); and then tuning the LayerNorm layers while other parameters are fixed in a round (i.e., \cref{eq:loss_diff} Parameter adaptation gain term). Such a round can be repetitive, lasting several turns to ensure the convergence of the overall model. 
{Following~\cite{kumar2022fine}, we show this approach is feasible.}

\begin{proposition}
\label{prop:lp_bounded}
   Let: $\mathcal{M}^{S}$ be a source model, and $\mathcal{C}^{S}$ be an MLP-based linear predictor (LP) paired with $\mathcal{M}^{S}$ such that the loss $L(\mathcal{C}^{S}, \mathcal{M}^{S}(X^{S}), Y^{S})$ is minimized (where $L(\cdot)$ is the Mean Squared Error). Assume $\mathcal{C}^{S}$ is not available during fine-tuning. When freezing $\mathcal{M}^{S}$ (e.g., a frozen ViT backbone), there exists a target predictor $\mathcal{C}^{T'}$ such that $L(\mathcal{C}^{S}, \mathcal{M}^{S}(X^{T}), Y^{T})$ is minimized. Then, the distance $d(\mathcal{C}^{S}, \mathcal{C}^{T'})$ is bounded.
\end{proposition}
\begin{proof}
    If $X^{S}$ and $X^{T}$ have no shift, \cref{prop:lp_bounded} is obvious. We mainly discuss whether there is a shift between them. 
    Considering the extracted features   $\mathcal{Z}$ of $\mathcal{M}^{S}(X)$ and the raw data space $\mathcal{X}$. 
    Since we have a well pretrained $\mathcal{M}^{S}$, following \cref{prop:lp_bounded}, we have 
    $\cos \theta_{\max}\left(\mathcal{Z}, \mathcal{X}\right)>0$. 
    Denote parameters of $\mathcal{C}^{T'}$ as $v_{\mathrm{lp}}^{T'}$, we have
    \begin{equation}
    \label{eq:L_ood_bound}
        {L\left(v_{\mathrm{lp}^{T'}}, Z^{T}\right)}^{1/2} \leq\left({c}/{\cos \theta_{\max }(\mathcal{Z},\mathcal{X} )}\right)^2 d\left(Z^{T}, Z^{S}\right)\| {{Z}^{S}}^{ \top }v_{\mathrm{lp}^{S}} \|_2^2,
    \end{equation}
    where $c > 0$ is a constant; $d(\cdot, \cdot)$ is a distance measurement. 
    Note \cref{eq:L_ood_bound} holds when the number of dimension of $v_{\mathrm{lp}}$ when $i \rightarrow {\infty}$. {Please refer to more details in \cref{app:math_details}.}
\end{proof}
Therefore, without fine-tuning LayerNorm, it is able to train $\mathcal{C}^{T'}$ with limited shifts from $\mathcal{C}^{S}$ using $X^{T}$. 
Then, we freeze the layers except the LayerNorm layers to minimize the \cref{eq:L_ood_bound}, where the problem turns into $\min_{\mathcal{LN}^{T'}}{L\left(\mathcal{LN}^{T'}, Z^{T}\right)}$ which is equivalent to:
\begin{equation}
    \label{eq:L_ood_bound_layernorm}
    \min_{\mathcal{LN}^{T'}} \left({c}/{\cos \theta_{\max }(\mathcal{M}^{T'}(\mathcal{X}),\mathcal{X} )}\right)^2 d\left(Z^{T}, \mathcal{M}^{T'}(X^{T})\right)\| \mathcal{M}^{T'}(X^{T})^{ \top }v_{\mathrm{lp}}^{T'} \|_2^2,
\end{equation}
where $\mathcal{LN}^{T'}, \mathcal{M}^{T'}$ correspond to the LayerNorm layers and the overall model tuned with $\mathcal{C}^{T'}$. It is obvious that the tuning LayerNorm layers do not change the knowledge obtained by the overall model therefore, the $\cos \theta_{\max }(\mathcal{M}^{T'}(\mathcal{X}),\mathcal{X} )$ can be considered to be as very close to the $\cos \theta_{\max }(\mathcal{M}^{S}(\mathcal{X}),\mathcal{X} )$, so does 
$\|\mathcal{M}^{T'}(X^{T})^{ \top }v_{\mathrm{lp}}^{T'} \|_2^2$ 
to $\|{Z^{S}}^{T} v_{\mathrm{lp}}^{T'} \|_2^2$.
\red{As discussed in \cref{sec:Tuning LayerNorm layers in the model collectively capture the data distribution shifts},  LayerNorm fine-tuning captures the data distributional shifts, hence $d\left(Z^{T}, \mathcal{M}^{T'}(X^{T})\right)\| \mathcal{M}^{T'}(X^{T}) \leq d\left(Z^{T}, Z^{S})\right)\| \mathcal{M}^{T'}(X^{T})$.} 
{Therefore, the bounds of the loss become tighter than \cref{eq:L_ood_bound_layernorm} during the the increasing number of turns.} This indicates why the proposed method as shown in \cref{fig:training} is feasible and stable. 
To further enhance LayerNorm fine-tuning, we introduce two additional techniques: (i) increasing the feature dimension before prediction and (ii) applying a lightweight feature augmentation prior to attention pooling. Details and ablations of them can be seen in \cref{app:More details about the proposed method}.


\textbf{Differences between our proposed method and LP-FT~\cite{kumar2022fine}.} 
\cite{kumar2022fine} proposes an LP-FT strategy, where the linear-probability (LP) predictor is fine-tuned first, followed by fine-tuning of the entire model, including all layers in $\mathcal{M}$ and the predictor. However, our experiments, as shown in \cref{fig:path_diff_ft_ablation}, indicate that this approach is generally ineffective for most ViTF fine-tuning scenarios. Given the typically limited amount of target training data, fine-tuning the full model, even with the LP-FT strategy, often leads to notable performance degradation. Moreover, when the LP-FT strategy is applied with fine-tuning restricted to LayerNorm, performance still deteriorates due to a violation of the $\mathcal{C}$-fixing presumption, as discussed in \cref{sec:Tuning LayerNorm layers in the model collectively capture the data distribution shifts}.

\section{Experiments}
\label{sec:exp}
\subsection{LayerNorm fine-tuning for natural images ViTFs:
ID, OOD, and complex settings}
\label{sec:ID LayerNorm fine-tuning vs OOD LayerNorm fine-tuning}

\begin{table}[t]
\caption{The results of (1) ID and OOD LayerNorm fine-tuning using MAE and (2) fine-tuning results of CLIP and DINOv2 on the DomainNet dataset. Domains are split as:
A = $\{Clipart, Infograph, Quickdraw \} $.
B = $\{Real, Painting, Sketch \}$.  
Setting \textbf{* to *}: pretrained on the former domain set, then fine-tuned and evaluated on the latter one.
Setting \textbf{To *, test on *}: fine-tuned on the former domain set, then evaluated on the latter one.
}
\label{tab:ID_OOD_res}
\begin{center}
\resizebox{0.85\textwidth}{!}{%
\begin{tabular}{cccc c >{\columncolor[HTML]{DAE8FC}}c >{\columncolor[HTML]{DAE8FC}}c  >{\columncolor[HTML]{DAE8FC}}c  |cccc c >{\columncolor[HTML]{DAE8FC}}c >{\columncolor[HTML]{DAE8FC}}c  >{\columncolor[HTML]{DAE8FC}}c }
\toprule[1.5pt]
\multicolumn{16}{c}{\textbf{MAE}} \\ \midrule
ID: A to A & LP+FM & LP & LP+LN & \cellcolor[HTML]{DAE8FC}\textbf{Ours} & Best $\lambda$ & +best $\lambda$ & {+$\lambda=1.2$} & OOD: A to B & LP+FM & LP & LP+LN & \cellcolor[HTML]{DAE8FC}\textbf{Ours} & Best $\lambda$ & +best $\lambda$ &+$\lambda\!=\!1.2$ \\
$1\%$ & 24.98 & 43.29 & 45.32 & \cellcolor[HTML]{DAE8FC}54.79 & 1.3 & 54.92 & {54.87} & $1\%$ & 24.03 & 39.8 & 40.21 & \cellcolor[HTML]{DAE8FC}40.28 & 2 & 41.76 & 40.82 \\
$5\%$ & 49.83 & 59.56 & 63.39 & \cellcolor[HTML]{DAE8FC}68.26 & 1 & 68.26 & {68.02} & $5\%$ & 48.97 & 56.56 & 57.27 & \cellcolor[HTML]{DAE8FC}57.8 & 1.2 & 57.95 & 57.95 \\
$10\%$ & 66.18 & 64.01 & 66.13 & \cellcolor[HTML]{DAE8FC}74.38 & 1.5 & 74.68 & {74.53} & $10\%$ & 58.06 & 58.8 & 61.01 & \cellcolor[HTML]{DAE8FC}61.16 & 1.5 & 61.43 & 61.38 \\
$50\%$ & 77.99 & 71.62 & 76.32 & \cellcolor[HTML]{DAE8FC}80.87 & 1.3 & 81.09 & {80.99} & $50\%$ & 69.56 & 70.13 & 69.47 & \cellcolor[HTML]{DAE8FC}69.81 & 1.3 & 70.08 & 70.07 \\
Avg & 54.75 & 59.62 & 62.79 & \cellcolor[HTML]{DAE8FC}69.58 & - & 69.74 & {69.6} & Avg & 50.16 & 56.32 & 56.99 & \cellcolor[HTML]{DAE8FC}57.26 & - & 57.81 & 57.56 \\ \hline
ID: B to B & LP+FM & LP & LP+LN & \cellcolor[HTML]{DAE8FC}\textbf{Ours} & Best $\lambda$ & +best $\lambda$ & {+$\lambda=1.2$} & OOD: B to A & LP+FM & LP & LP+LN & \cellcolor[HTML]{DAE8FC}\textbf{Ours} & Best $\lambda$ & +best $\lambda$ &+$\lambda\!=\!1.2$ \\
$1\%$ & 23.34 & 45.48 & 45 & \cellcolor[HTML]{DAE8FC}45.4 & 1.3 & 45.5 & {45.47} & $1\%$ & 24.55 & 43.76 & 45.4 & \cellcolor[HTML]{DAE8FC}45.51 & 1.2 & 45.49 & 45.49 \\
$5\%$ & 47.47 & 59.33 & 60.57 & \cellcolor[HTML]{DAE8FC}60.59 & 0.8 & 60.95 & {60.38} & $5\%$ & 24.55 & 59.22 & 63.11 & \cellcolor[HTML]{DAE8FC}62.81 & 1.1 & 62.87 & 62.86 \\
$10\%$ & 60.65 & 64.81 & 65.75 & \cellcolor[HTML]{DAE8FC}66.35 & 1.2 & 66.49 & {66.49} & $10\%$ & 65.21 & 64 & 66.35 & \cellcolor[HTML]{DAE8FC}67.71 & 1.5 & 67.92 & 67.86 \\
$50\%$ & 69.68 & 70.55 & 72.89 & \cellcolor[HTML]{DAE8FC}73.67 & 1.5 & 73.9 & {73.88} & $50\%$ & 78.7 & 72.28 & 76.25 & \cellcolor[HTML]{DAE8FC}77.52 & 1.3 & 77.62 & 77.61 \\
Avg & 50.29 & 60.04 & 61.05 & \cellcolor[HTML]{DAE8FC}61.5 & - & 61.71 & {61.56} & Avg & 48.25 & 59.82 & 62.78 & \cellcolor[HTML]{DAE8FC}63.39 & - & 63.48 & 63.46 \\ \midrule
\multicolumn{16}{c}{\textbf{CLIP} (ID/OOD settings are unknown)} \\ \midrule
To A, test on A & LP+FM & LP & LP+LN & \cellcolor[HTML]{DAE8FC}\textbf{Ours} & Best $\lambda$ & +best $\lambda$ & {+$\lambda=1.2$} & To A, test on B & LP+FM & LP & LP+LN & \cellcolor[HTML]{DAE8FC}\textbf{Ours} & Best $\lambda$ & +best $\lambda$ &+$\lambda\!=\!1.2$ \\
$1\%$ & 36.98 & 75.45 & 77.86 & \cellcolor[HTML]{DAE8FC}88.13 & 2 & 88.54 & {88.28} & $1\%$ & 17.84 & 42.72 & 45.71 & \cellcolor[HTML]{DAE8FC}61.29 & 2 & 62.16 & 61.36 \\
$5\%$ & 62.74 & 92.31 & 93.24 & \cellcolor[HTML]{DAE8FC}91.38 & 1.3 & 91.41 & {91.38} & $5\%$ & 32.07 & 64.68 & 66.87 & \cellcolor[HTML]{DAE8FC}62.23 & 1.3 & 62.48 & 62.41 \\
$10\%$ & 83.35 & 50.41 & 57.29 & \cellcolor[HTML]{DAE8FC}93.45 & 1.5 & 93.48 & {93.5} & $10\%$ & 54.85 & 28.29 & 29.7 & \cellcolor[HTML]{DAE8FC}67.29 & 1.5 & 67.77 & 67.53 \\
$50\%$ & 86.78 & 90.49 & 92.25 & \cellcolor[HTML]{DAE8FC}94.44 & 1.1 & 94.45 & {94.41} & $50\%$ & 57.37 & 62.5 & 64.61 & \cellcolor[HTML]{DAE8FC}67.02 & 1.1 & 67.07 & 67.25 \\
Avg & 67.46 & 77.17 & 80.16 & \cellcolor[HTML]{DAE8FC}91.85 & - & 91.97 & {91.89} & Avg & 40.53 & 49.55 & 51.72 & \cellcolor[HTML]{DAE8FC}64.46 & - & 64.87 & 64.64 \\ \hline
To B, test on B & LP+FM & LP & LP+LN & \cellcolor[HTML]{DAE8FC}\textbf{Ours} & Best $\lambda$ & +best $\lambda$ & {+$\lambda=1.2$} & To B, test on A & LP+FM & LP & LP+LN & \cellcolor[HTML]{DAE8FC}\textbf{Ours} & Best $\lambda$ & +best $\lambda$ &+$\lambda\!=\!1.2$ \\
$1\%$ & 13.11 & 43.15 & 45 & \cellcolor[HTML]{DAE8FC}70 & 1.2 & 70.24 & {70.24} & $1\%$ & 9.98 & 54.48 & 52.54 & \cellcolor[HTML]{DAE8FC}79.76 & 1.2 & 79.86 & 79.86 \\
$5\%$ & 56.13 & 79.07 & 81.34 & \cellcolor[HTML]{DAE8FC}77.51 & 1 & 77.51 & {77.26} & $5\%$ & 41.71 & 84.64 & 86.65 & \cellcolor[HTML]{DAE8FC}78.98 & 1 & 78.98 & 79.02 \\
$10\%$ & 71.52 & 33.08 & 32.36 & \cellcolor[HTML]{DAE8FC}82.79 & 1.2 & 82.86 & {82.86} & $10\%$ & 67.3 & 39 & 36.66 & \cellcolor[HTML]{DAE8FC}89.61 & 1.2 & 89.67 & 89.67 \\
$50\%$ & 75.98 & 76.28 & 79.02 & \cellcolor[HTML]{DAE8FC}86.55 & 1.2 & 86.61 & {86.61} & $50\%$ & 67.5 & 85.81 & 86.81 & \cellcolor[HTML]{DAE8FC}90.18 & 1.2 & 90.11 & 90.11 \\
Avg & 54.19 & 57.9 & 59.43 & \cellcolor[HTML]{DAE8FC}79.21 & - & 79.31 & {79.24} & Avg & 46.62 & 65.98 & 65.67 & \cellcolor[HTML]{DAE8FC}84.63 & - & 84.66 & 84.67 \\ \midrule
\multicolumn{16}{c}{\textbf{DINOv2} (ID/OOD settings are unknown)} \\ \midrule
To A, test on A & LP+FM & LP & LP+LN & \cellcolor[HTML]{DAE8FC}\textbf{Ours} & Best $\lambda$ & +best $\lambda$ & {+$\lambda=1.2$} & To A, test on B & LP+FM & LP & LP+LN & \cellcolor[HTML]{DAE8FC}\textbf{Ours} & Best $\lambda$ & +best $\lambda$ &+$\lambda\!=\!1.2$ \\
$1\%$ & 72.76 & 80.33 & 83.16 & \cellcolor[HTML]{DAE8FC}90.29 & 1.3 & 90.37 & {90.36} & $1\%$ & 34.13 & 43.17 & 48.86 & \cellcolor[HTML]{DAE8FC}55.78 & 1.3 & 56.1 & 55.97 \\
$5\%$ & 90.21 & 92.34 & 93.71 & \cellcolor[HTML]{DAE8FC}93.48 & 1.3 & 93.56 & {93.5} & $5\%$ & 46.41 & 57.81 & 62.98 & \cellcolor[HTML]{DAE8FC}62.81 & 1.3 & 63.16 & 62.99 \\
$10\%$ & 90.21 & 73.75 & 84.19 & \cellcolor[HTML]{DAE8FC}93.74 & 1.3 & 93.82 & {93.75} & $10\%$ & 57.28 & 41.31 & 48.65 & \cellcolor[HTML]{DAE8FC}62.47 & 1.3 & 62.62 & 62.61 \\
$50\%$ & 93.79 & 92.44 & 94.71 & \cellcolor[HTML]{DAE8FC}95.79 & 1.3 & 95.97 & {95.92} & $50\%$ & 62.49 & 58.69 & 65.08 & \cellcolor[HTML]{DAE8FC}66.4 & 1.3 & 66.76 & 66.63 \\
Avg & 86.74 & 84.72 & 88.94 & \cellcolor[HTML]{DAE8FC}93.33 & - & 93.43 & {93.38} & Avg & 50.08 & 50.25 & 56.39 & \cellcolor[HTML]{DAE8FC}61.87 & - & 62.16 & 62.05 \\ \hline
To B, test on B & LP+FM & LP & LP+LN & \cellcolor[HTML]{DAE8FC}\textbf{Ours} & Best $\lambda$ & +best $\lambda$ & {+$\lambda=1.2$} & To B, test on A & LP+FM & LP & LP+LN & \cellcolor[HTML]{DAE8FC}\textbf{Ours} & Best $\lambda$ & +best $\lambda$ &+$\lambda\!=\!1.2$ \\
$1\%$ & 50.71 & 50.71 & 55.16 & \cellcolor[HTML]{DAE8FC}70.06 & 1.5 & 70.82 & {70.44} & $1\%$ & 62.17 & 48.03 & 62.5 & \cellcolor[HTML]{DAE8FC}80.14 & 1.5 & 80.45 & 80.34 \\
$5\%$ & 79.9 & 79.9 & 82.55 & \cellcolor[HTML]{DAE8FC}81.21 & 1.1 & 81.13 & {81.12} & $5\%$ & 83.71 & 70.14 & 87.11 & \cellcolor[HTML]{DAE8FC}86.3 & 1.1 & 86.4 & 86.47 \\
$10\%$ & 47.47 & 47.47 & 66.05 & \cellcolor[HTML]{DAE8FC}81.53 & 1.3 & 81.75 & {81.68} & $10\%$ & 49.53 & 83.46 & 68.03 & \cellcolor[HTML]{DAE8FC}90.12 & 1.3 & 90.31 & 90.29 \\
$50\%$ & 79.48 & 79.48 & 84.76 & \cellcolor[HTML]{DAE8FC}87.37 & 1.2 & 87.44 & {87.44} & $50\%$ & 85.51 & 83.67 & 90.33 & \cellcolor[HTML]{DAE8FC}91.25 & 1.2 & 91.28 & 91.28 \\
Avg & 64.39 & 64.39 & 72.13 & \cellcolor[HTML]{DAE8FC}80.04 & - & 80.29 & {80.17} & Avg & 70.23 & 71.33 & 76.99 & \cellcolor[HTML]{DAE8FC}86.95 & - & 87.11 & 87.1 \\ \bottomrule[1.5pt]
\end{tabular}%
}
\end{center}
\caption{Results of fine-tuning OpenCLIP and DINOv2 on SUN datasets. 
}
\label{tab:natural_image_res}
\centering
\resizebox{0.85\textwidth}{!}{%
\begin{tabular}{cccc
>{\columncolor[HTML]{DAE8FC}}c 
>{\columncolor[HTML]{DAE8FC}}c 
>{\columncolor[HTML]{DAE8FC}}c |cccc
>{\columncolor[HTML]{DAE8FC}}c 
>{\columncolor[HTML]{DAE8FC}}c 
>{\columncolor[HTML]{DAE8FC}}c }
\toprule[1.5pt]
\textbf{OOD: OpenCLIP To SUN} & LP+FM & LP & LP+LN & Ours & Best $\lambda$ value & +best $\lambda$ & \textbf{ID: DINOv2 To SUN} & LP+FM & LP & LP+LN & Ours & Best $\lambda$ value & +best $\lambda$ \\
$2\%$ & 34.78 & 20.97 & 32.27 & 35.49 & 1.0 & 35.49 & $2\%$ & 46.80 & 45.23 & 43.95 & 45.16 & 0.2 & 45.51 \\
$5\%$ & 58.24 & 51.29 & 59.35 & 58.70 & 1.5 & 59.06 & $5\%$ & 62.62 & 62.33 & 63.13 & 63.32 & 0.2 & 63.38 \\
$10\%$ & 64.30 & 61.96 & 67.30 & 66.13 & 1.1 & 66.15 & $10\%$ & 66.56 & 67.47 & 67.86 & 68.52 & 0.2 & 68.52 \\
$100\%$ & 73.68 & 76.13 & 75.19 & 76.70 & 1.3 & 76.80 & $100\%$ & 74.24 & 77.13 & 77.22 & 77.84 & 0.3 & 77.96 \\
Avg & 57.75 & 52.59 & 58.53 & 59.25 & - & 59.37 & Avg & 62.56 & 63.04 & 63.04 & 63.71 & - & 63.84 \\ \midrule
\textbf{OOD: OpenCLIP To DTD} & LP+FM & LP & LP+LN & Ours & Best $\lambda$ value & +best $\lambda$ & \textbf{ID: DINOv2 To DTD} & LP+FM & LP & LP+LN & Ours & Best $\lambda$ value & +best $\lambda$ \\
$2\%$ & 0.00 & 0.00 & 0.00 & 6.38 & 1.2 & 6.38 & $2\%$ & 2.13 & 2.13 & 2.13 & 4.26 & 0.2 & 4.26 \\
$5\%$ & 32.98 & 11.70 & 37.23 & 45.74 & 1.2 & 46.81 & $5\%$ & 5.32 & 55.32 & 57.45 & 63.83 & 0.2 & 63.83 \\
$10\%$ & 50.53 & 25.00 & 47.34 & 51.06 & 1.2 & 51.06 & $10\%$ & 4.26 & 56.91 & 59.57 & 61.17 & 0.2 & 61.17 \\
$100\%$ & 48.83 & 79.47 & 79.20 & 79.89 & 1.2 & 80.05 & $100\%$ & 15.64 & 78.99 & 81.81 & 80.59 & 2.0 & 81.12 \\
Avg & 33.09 & 29.04 & 40.94 & 45.77 & - & 46.08 & Avg & 6.84 & 48.34 & 50.24 & 52.46 & - & 52.59 \\ \midrule
\textbf{OOD: OpenCLIP To Bach} & LP+FM & LP & LP+LN & Ours & Best $\lambda$ value & +best $\lambda$ & \textbf{OOD: DINOv2 To Bach} & LP+FM & LP & LP+LN & Ours & Best $\lambda$ value & +best $\lambda$ \\
1 & 33.50 & 44.50 & 31.50 & 37.00 & 1.2 & 37.50 & 1 & 27.00 & 32.50 & 39.50 & 37.50 & 1.3 & 37.50 \\
5 & 49.50 & 50.00 & 46.50 & 52.00 & 1.3 & 55.00 & 5 & 25.00 & 51.50 & 52.00 & 53.50 & 1.5 & 59.00 \\
10 & 50.00 & 49.00 & 54.00 & 59.00 & 1.2 & 61.50 & 10 & 25.00 & 62.00 & 59.50 & 61.50 & 1.1 & 60.50 \\
20 & 43.00 & 47.50 & 50.00 & 60.50 & 1.1 & 60.50 & 20 & 29.50 & 61.00 & 63.00 & 61.50 & 1.0 & 61.50 \\
All & 66.50 & 63.50 & 67.00 & 70.50 & 1.2 & 71.00 & All & 45.00 & 74.00 & 74.50 & 76.50 & 0.6 & 77.00 \\
Avg & 48.50 & 50.90 & 49.80 & 55.80 & - & 57.10 & Avg & 30.30 & 56.20 & 57.70 & 58.10 & - & 59.10 \\ \bottomrule[1.5pt]
\end{tabular}%
}
\vspace{-0.5cm}
\end{table}

\textbf{Experimental settings.}
To better understand the behavior of LayerNorm fine-tuning, we conduct experiments by pretraining an unsupervised MAE~\cite{he2022masked} and subsequently fine-tuning it on both in-domain (ID) and out-of-domain (OOD) samples (domain split details can be seen in \cref{tab:ID_OOD_res}).
Specifically, we use the natural image dataset, DomainNet dataset~\cite{peng2019moment}, which contains six distinct domains: Real, Painting, Sketch, Clipart, Infograph, and Quickdraw. 
The MAE is pre-trained and fine-tuned using ID settings (same domains) and OOD settings (unseen domains) of the DomainNet dataset.
For comparison, we also apply LayerNorm fine-tuning to the large-scale pretrained ViTFs of natural images, where the domain shift between the source and target is difficult to determine. Here, we employ the most commonly used natural image ViTFs, including CLIP~\cite{radford2021learning} (ViT-B/32), OpenCLIP~\cite{radford2021learning} (ViT-B/32 pretrained on DataComp-1B), and DINOv2~\cite{oquab2023dinov2} (pretrained on LVD-142M) on the DomainNet domains.
Additionally, we employ natural image datasets, SUN~\cite{xiao2010sun} and DTD~\cite{cimpoi2014describing} (included by LVD-142M for DINOv2), accompanied by the pathological dataset, Bach~\cite{aresta2019bach}, for testing the ID and OOD settings for natural image ViTFs. 
We fine-tune all models using varying fractions of labeled target-domain data ($1\%, 5\%, 10\%, 50\%$) and evaluate on held-out test sets.
Evaluation results are taken from the final epoch, and $\lambda$ is further chosen based on test set performance as the indicator of $FSR$.
See \cref{app:Natural image experimental details} for full details.

\begin{wrapfigure}{r}{0.33\linewidth}
    \centering
  \includegraphics[width=\linewidth]{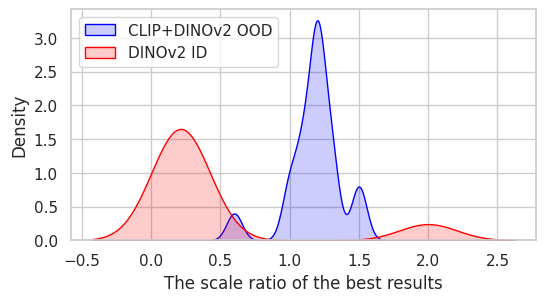}
    \caption{
    Kernel density estimation visualization of the optimal (best) values of $\lambda$ from  \cref{tab:natural_image_res} for large-scale pretrained ViTFs across different settings. 
    }
  \label{fig:ID_OOD_ratio_density_2}
\vspace{-0.3cm}
\end{wrapfigure}
\textbf{Comparisons.}
For comparison, we additionally conduct 
(1) LP+FM: Fine-tuning the full $\mathcal{M}$ with LP,
(2) LP: Training the LP model with the frozen $\mathcal{M}$,
and (3) LP+LN: Fine-tuning the $\mathcal{LN}$ with LP simultaneously.
All experiments, including our method and the comparisons, use the same overall training epochs and learning rates for a fair comparison.

\textbf{Results.} 
Our experimental results in \cref{tab:ID_OOD_res,tab:natural_image_res} reveal several key insights regarding LayerNorm fine-tuning. First, both the ID and OOD settings significantly impact the final performance, as does the amount of available fine-tuning data. 
Specifically, using ID or a larger amount of target training data would lead to better performance. 
Despite these variations, LayerNorm fine-tuning consistently performs well across different training data fractions in both ID and OOD scenarios. Moreover, our proposed method outperforms standard LayerNorm fine-tuning across all settings. 
Interestingly, we observe that the optimal scaling parameter $\lambda$ exhibits a mild correlation with the ID or OOD setting, although the domain shift between source and target does not directly determine it. As shown in \cref{fig:banner}~(c), the distribution of optimal $\lambda$ in the MAE ID setting shows a left-skewed tail, while the OOD setting displays a right-skewed tail. This suggests that in OOD scenarios, the limited target training data is more likely to be insufficient to fully represent the target domain distribution, leading to a larger $FSR$. 
Moreover, the observed $\lambda$ distributions for both CLIP and DINOv2 resemble those typically associated with OOD scenarios, aligning with the fact that DomainNet was not included in the pretraining datasets of either model.


\textbf{Additional results of large-scale pretrained ViTFs.} 
\cref{tab:natural_image_res} shows that our proposed LayerNorm fine-tuning method consistently outperforms baseline approaches, particularly under limited labeled target data. 
It is worth noticing that the LP-FM may experience severe collapse if there are very limited target training data.  
Moreover, rescaling using $\lambda$ may need further improvements. 
As indicated in \cref{fig:ID_OOD_ratio_density_2},
for SUN and DTD datasets, the optimal $\lambda$ values differ significantly between the two models: OpenCLIP tends to favor $\lambda \ge 1.0$, whereas DINOv2 achieves optimal performance with $\lambda < 1.0$. 
This trend remains consistent across all fractions of labeled target data, suggesting that OpenCLIP exhibits a smaller $FSR$, indicative of an OOD setting, in contrast to DINOv2 on the SUN dataset, which reflects an ID scenario. 
For the Bach dataset, both models obtain relatively larger $\lambda$. Moreover, $\lambda$ is reduced for \textit{OOD: DINOv2 To Bach} but increased for \textit{ID: DINOv2 To DTD} when more target training data are used, indicating that the $FSR$ tends to be close to $1$.
Thus, though the source data distribution may be intractable, it is possible to use the optimal value of $\lambda$ to reflect the shift of the target training data compared to the overall testing data from the source domain.
\textbf{Therefore, we recommend a $\lambda \ge 1$ (such as $1.2$) for most natural image fine-tuning settings. However,
suppose there is a certain prior knowledge of whether it is an ID or OOD setting. In that case, we recommend setting $\lambda$ larger for the OOD setting and $\lambda$ smaller for the ID setting, especially where there are insufficient target training samples.}
More results and analysis are in \cref{app:Natural image experimental details}.






\subsection{LayerNorm fine-tuning for Pathology images ViTFs}
\label{sec:path_exp}

\textbf{Experimental settings.}
We evaluate the fine-tuning performance of pathological ViTFs (CONCH~\cite{lu2024visual}, CHEIF~\cite{wang2024pathology}, and TITAN~\cite{ding2024multimodal})
using five benchmark datasets: Bach~\cite{aresta2019bach}, Breakhis~\cite{spanhol2015dataset}, CCRCC~\cite{brummer2023computational} (including both binary and tissue subsets), and Glas~\cite{sirinukunwattana2017gland}. For each dataset, we randomly split the labeled data in half, using one half exclusively for testing. The target training samples are randomly drawn from the remaining half. Specifically, we experimented with $1$, $5$, $10$, and $20$ samples per class, and we additionally included experiments using all available target training samples for Bach and Glas. A fixed random seed $42$ is used for all data splits. The experimental settings and baseline methods remain consistent with those used for natural image datasets, except for the omission of LP+FM tuning due to its consistently poor performance, particularly in limited target training samples scenarios, as demonstrated in \cref{tab:ID_OOD_res,tab:natural_image_res}.  Similar to natural image ViTFs,  $\lambda$ is further chosen based on test set performance as the indicator of $FSR$.

\textbf{Results.}
The results illustrated in \cref{fig:path_res} demonstrate the consistent superiority of our proposed method and the impact of rescaling $\gamma$ using $\lambda$ across multiple datasets and pathological ViTFs. 
Compared to baseline approaches, LP and LP+LN, our approach achieves higher accuracy in most settings, highlighting its strong feasibility to varying data distributions. 

\begin{wrapfigure}{r}{0.48\linewidth}
\centering
  \includegraphics[width=\linewidth]{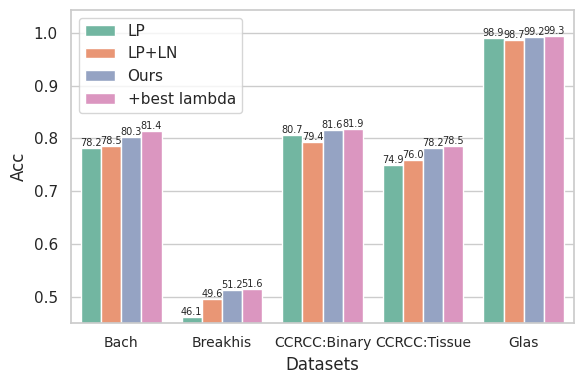}
    \caption{Results of fine-tuning pathological ViTFs across datasets. Note here that all results are averaged across different numbers of target training samples per class and various pathological ViTFs. Please refer to more details for each model in Appendix~\cref{fig:path_res_details}.
    }
  \label{fig:path_res}

  \centering
  \includegraphics[width=\linewidth]{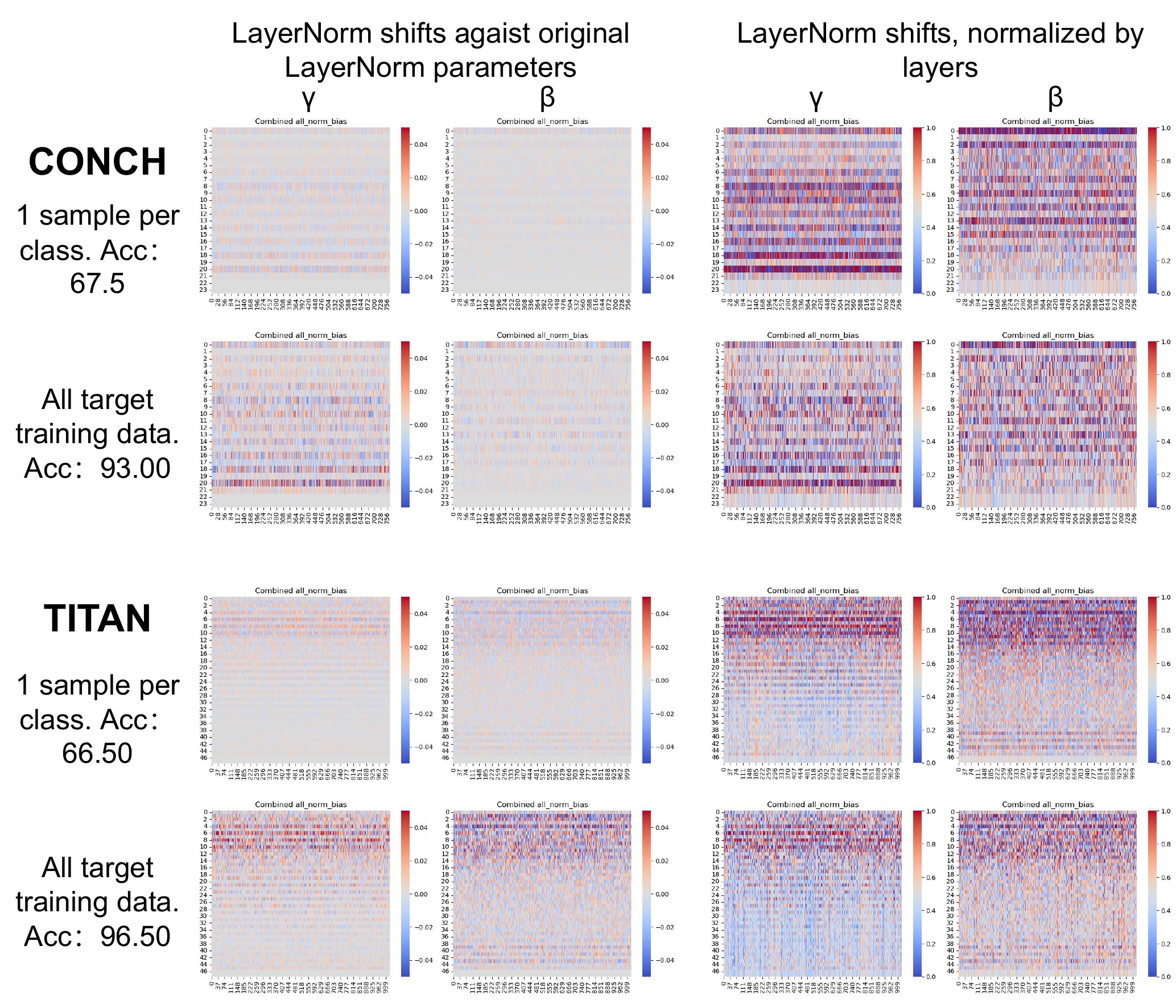}
    \caption{
    Visualizations of LayerNorm shifts on the Bach dataset. See more details in Appendix~\cref{fig:path_norm_vis_detailed}.
    }
  \label{fig:path_norm_vis}
\vspace{-0.5cm}
\end{wrapfigure}
Moreover, the additional performance gains achieved through rescaling align with our analysis in \cref{sec:Rescaling LayerNorm shifts leads to improvements}. As shown in \cref{fig:banner}~(d), the optimal $\lambda$ values for ViTFs on pathological images are notably smaller than those on natural images, with a clear gap between the two. {We hypothesize that this is due to more pronounced domain shifts in natural images, whereas pathological images exhibit subtler variations, resembling an ID setting with a relatively larger $FSR$ under most scenarios.}  
\textbf{Accordingly, without prior knowledge, we recommend setting $\lambda \le 1$ for pathological ViFTs fine-tuning on pathological images.} More results of our proposed LayerNorm fine-tuning method can be seen in \cref{app:LayerNorm fine-tuning for Pathology images ViTFs: Experimental details and more results}.

\textbf{Other findings in LayerNorm fine-tuning.}
\Cref{fig:path_norm_vis} visualizes the shifts in all LayerNorm parameters, the scale $\gamma$ and bias $\beta$ within the $\mathcal{M}$, where even-numbered layers correspond to pre-attention LayerNorm and odd-numbered layers to post-attention LayerNorm. The visualization reveals several key insights. 
(1)~Pre-attention LayerNorms exhibit greater shifts and thus contribute more significantly to fine-tuning performance than their post-attention counterparts. 
(2)~The specific LayerNorm layers requiring tuning vary across different ViT architectures, indicating model-dependent behavior rather than a consistent relationship with the number of target training samples. Thus, tuning all LayerNorm layers would be a robust approach for all ViTFs. 
(3)~The shifts in $\gamma$ and $\beta$ are not always synchronized across layers (e.g., for CONCH), suggesting that scale and bias may play distinct roles during adaptation. (4)~Lastly, increasing the number of tuning samples generally leads to sparser LayerNorm updates and better performance, implying greater confidence in fewer parameters. 
Sparsifying LayerNorm updates in low-data regimes can be beneficial but is challenging due to model-specific complexities and sample selection uncertainties. More attempts and the corresponding results can be seen in \cref{app:Attempts for sparsing the LayerNorm shifts while there are fewer target training samples}.

\section{Conclusion and discussions}
\textbf{Discussions.}
This paper studies the interaction between  LayerNorm fine-tuning and different amounts of target training data or with domain shifts through extensive analysis and experiments. 
Our findings also highlight the importance of selecting representative fine-tuning samples: when the sufficiency of target training data is uncertain, one can rescale the model across a range of $\lambda$ values and use the optimal $\lambda$ as an indicator of data adequacy.
Currently, our study primarily focuses on the visual domain, where domain shifts are more readily quantifiable and interpretable. We hypothesize that similar phenomena may arise in language transformers; however, validating this would require careful experimental design, which we leave for future work. Additionally, while our analysis centers on standard LayerNorm, alternative normalization variants~\cite{xu2019understanding} may exhibit analogous behaviors, and architectures such as DiT~\cite{Peebles_2023_ICCV} or transformers without explicit normalization layers~\cite{zhu2025transformers} fall outside the scope of our current investigation. 
\textbf{Limitations.} 
Our study only focuses on classification with linear predictors to isolate LayerNorm effects; extensions to more complex tasks and models are discussed in \cref{app:Limitations}.


{
\small
\bibliographystyle{plain}
\bibliography{main.bib}
}

\appendix

\section{Limitations}
\label{app:Limitations}
Our study is limited to classification tasks, where the predictor $\mathcal{C}$ is linear, in order to isolate the effects of LayerNorm fine-tuning and avoid confounding factors introduced by more complex predictors. Extending our framework to other tasks, such as detection, segmentation, or multi-task learning, remains an important direction for future work. In multi-task scenarios, it is plausible that different LayerNorm scale parameters ($\gamma$) may require distinct rescaling directions, i.e., some increasing, others decreasing, depending on the specific sub-task characteristics. Moreover, our approach explicitly assumes that the source model $\mathcal{M}^{S}$ is well-trained and capable of extracting all necessary information from $X$ for predicting $Y$. If the target task involves novel label spaces or requires information not captured by $\mathcal{M}^{S}$, fine-tuning only the LayerNorm layers would likely be insufficient. Nevertheless, in such cases, combining LayerNorm tuning with additional mechanisms that enrich the model's feature representations could further enhance downstream performance.

\section{Social impact}
This work aims to deepen our understanding of LayerNorm fine-tuning in vision transformer foundation models, with a focus on its behavior under different data availability and distribution shift scenarios. By providing both empirical insights and theoretical considerations, the study supports the development of more efficient and effective adaptation strategies for large-scale models, potentially lowering the barriers to entry for research groups or practitioners with limited computational resources or labeled data. 
However, parameter-efficient adaptation methods also introduce potential risks. When models are fine-tuned with minimal supervision, it becomes harder to ensure robustness and fairness across unseen domains or subpopulations, particularly in high-stakes applications such as healthcare or surveillance. Our results highlight that the reliability of such tuning methods can vary significantly depending on data distribution and quantity, suggesting that blind adoption may lead to misleading performance estimations.

We believe that while LayerNorm fine-tuning holds promise for scalable and efficient model adaptation, its application in socially sensitive areas, such as medical diagnosis, surveillance, or automated decision-making, should be approached cautiously. We encourage future research to integrate stronger evaluation protocols and uncertainty estimation tools when applying lightweight fine-tuning in real-world deployments, particularly in sensitive or socially consequential contexts.

\section{Related work}
\label{sec:Related work}
\textbf{Studies of normalizations.}
Normalization layers have become integral components of deep learning architectures across both language and vision domains. The seminal work on Batch Normalization~\cite{ioffe2015batch} introduced normalization as a means to mitigate internal covariate shift, thereby improving training stability and convergence. Since then, various normalization techniques have been proposed, including Instance Normalization~\cite{ulyanov2016instance}, Group Normalization~\cite{wu2018group}, and Layer Normalization~\cite{ba2016layer}. Among these, LayerNorm has emerged as the standard choice in the design of large-scale models, particularly in large language models (LLMs) and visual transformer foundation models (ViTFs). In this work, we focus specifically on the role of LayerNorm within ViTFs.

\textbf{Parameter efficient
transfer learning.} 
Parameter-Efficient Transfer Learning.
Our work is closely related to parameter-efficient transfer learning, within which LayerNorm fine-tuning represents a specific subcase. Broadly, parameter-efficient methods either introduce additional lightweight trainable components, such as adapters~\cite{houlsby2019parameter} or prompt tokens~\cite{jia2022visual}, into a frozen ViTF, or selectively fine-tune a subset of the model’s existing parameters, such as all bias terms~\cite{zaken2021bitfit}. More recently, LoRA~\cite{hu2022lora} has demonstrated that optimizing low-rank decomposition matrices for weight updates in dense layers offers an effective and scalable approach to adapting large models.

\textbf{LayerNorm fine-tuning approaches.}
For LLMs,
\cite{zhao2023tuning} has demonstrated that LayerNorm fine-tuning is both effective and sufficient for multi-modal LLMs, outperforming widely used methods such as LoRA~\cite{hu2022lora}, a finding further supported by~\cite{chen2024efficiency}. Similarly, \cite{valizadehaslani2024layernorm} emphasizes the centrality of LayerNorm in parameter-efficient fine-tuning strategies. In the context of ViTFs, \cite{de2023effectiveness} shows that LayerNorm tuning is likewise efficient and proposes introducing auxiliary learnable parameters to capture LayerNorm shifts instead of modifying the original parameters directly. Despite these promising results, existing studies have not sufficiently examined the underlying mechanisms of LayerNorm during fine-tuning, nor have they explored how its effectiveness might be further enhanced.
 
\textbf{Visual transformer foundation models.}
Visual transformer foundation models (ViTFs) can broadly be categorized into two main branches. The first focuses on aligning multi-modal representations, particularly between text and images, as exemplified by models such as CLIP and OpenCLIP~\cite{radford2021learning}. These multi-modal ViTFs have extended their influence beyond natural images, gaining traction in medical imaging domains, like pathological images, through models like CONCH~\cite{lu2024visual}, CHEIF~\cite{wang2024pathology}, and TITAN~\cite{ding2024multimodal}. The second branch comprises models trained exclusively on visual modalities, such as DINOv2~\cite{oquab2023dinov2}, which leverage self-supervised learning frameworks. In particular, methods like Masked Autoencoders (MAE)~\cite{he2022masked} have demonstrated strong potential in advancing the training of ViTFs within this branch.

\section{Mathematical details}
\label{app:math_details}

\textbf{Proof details of \cref{prop:shift_compare}}
\red{To achieve that, we connect $\mathcal{LN}^{S}$ and $\mathcal{LN}^{T}$ by introducing definitions and assumptions on the loss function.}
The overall loss of a model on a given dataset $(X, Y)$, with LayerNorm parameters $\mathcal{LN}$, is defined as:
\begin{equation}
    L(\mathcal{LN}, X,Y) = \mathbb{E}_{(x,y) \sim (X,Y)} [\ell ( \mathcal{C}(\mathcal{M}(x)), y)],
\end{equation}
where $\mathcal{M}$ is the model with fixed parameters except for $\mathcal{LN}$, and $\mathcal{C}$ denotes the predictor.
Given an original model trained on the source dataset $(X^S, Y^S)$ and a fine-tuned model on the target dataset $(X^T, Y^T)$, we compare their losses to capture the relationship between $\mathcal{LN}^{S}$ and $\mathcal{LN}^{T}$ as follows:
\begin{equation}
\label{eq:loss_diff}
\begin{split}
    &\Delta L = L(\mathcal{LN}^{S}, X^{S}, Y^{S}) - L(\mathcal{LN}^{T}, X^{T}, Y^{T})  \\ 
   &= \underbrace{L\left(\mathcal{LN}^{S}, X^{S}, Y^{S}\right) - L\left(\mathcal{LN}^{S}, X^{T}, Y^{T}\right)}_{\text{Domain gap loss}} 
   + \underbrace{L\left(\mathcal{LN}^{S}, X^{T}, Y^{T}\right) - L\left(\mathcal{LN}^{T}, X^{T}, Y^{T}\right)}_{\text{Parameter adaptation gain}}.
\end{split}
\end{equation}
Assuming both models are well-optimized on their respective datasets, i.e., $\Delta L \approx 0$ since
$L(\mathcal{LN}^{S}, X^{S}, Y^{S}) \approx 0$, $L(\mathcal{LN}^{T}, X^{T}, Y^{T}) \approx 0$. This implies that:
\begin{equation}
    L\left(\mathcal{LN}^{S}, X^{T}, Y^{T}\right) - L\left(\mathcal{LN}^{T}, X^{T}, Y^{T}\right) 
    \approx L\left(\mathcal{LN}^{S}, X^{T}, Y^{T}\right) -L\left(\mathcal{LN}^{S}, X^{S}, Y^{S}\right)  .
\end{equation}
That is, the performance gain from adapting $\mathcal{LN}$  to the target domain is roughly equal to the distribution shift between source and target data \textbf{under other parameters in the $\mathcal{M}$ and $\mathcal{C}$ are fixed}.
Thus, we have \cref{prop:shift_compare}.

\textbf{Details analysis of \cref{sec:Rescaling LayerNorm shifts leads to improvements}.}
One crucial problem here is that the convergence behavior of the normalization parameters $\gamma$ and $\beta$ differs depending on the number of target training samples $n$. 

In Layer Normalization, $\beta$ and $\gamma$ are learnable affine parameters applied after normalization. During adaptation or estimation under domain shift, they are often initialized or updated based on batch statistics: $\beta$ can be associated with the estimated mean $\mu$, while $\gamma$ may be related to the estimated standard deviation $\sigma$. 

Assuming the target data samples $X_1, ..., X_n$ are i.i.d. with unknown mean $\mu$ and variance $\sigma^2$, we have the sample mean as $\bar{X}_n = \frac{1}{n} \sum_{i=1}^{n} X_i$.
By the Central Limit Theorem (CLT), for large enough $n$, we have  $\bar{X}_n \sim \mathcal{N}(\mu, \frac{\sigma^2}{n})$ and the standard error of the mean as $\mathrm{SE}(\bar{X}_n) = \frac{\sigma}{\sqrt{n}}$.

To ensure the $(1-\alpha)$ confidence interval of $\mu$ has width no more than $\varepsilon$:
\begin{equation}
   2 z_{\alpha/2} \cdot \frac{\sigma}{\sqrt{n}} \le \varepsilon, 
\end{equation}

where we require:
\begin{equation}
n \ge \left( \frac{2 z_{\alpha/2} \sigma}{\varepsilon} \right)^2.
\end{equation}

In contrast, estimating the variance $\sigma^2$ (as a proxy for $\gamma^2$ in LayerNorm) converges more slowly. Under the assumption that the data follows a Gaussian distribution, the sample variance $\hat{\sigma}^2$ satisfies:
\begin{equation}
\frac{(n - 1)\hat{\sigma}^2}{\sigma^2} \sim \chi^2_{n - 1}.
\end{equation}

Hence, the $(1 - \alpha)$ confidence interval for $\sigma^2$ is:
\begin{equation}
\left[
\frac{(n - 1)\hat{\sigma}^2}{\chi^2_{1 - \alpha/2}},\ 
\frac{(n - 1)\hat{\sigma}^2}{\chi^2_{\alpha/2}}
\right].
\end{equation}

To ensure the confidence interval width is within $2\varepsilon$, we solve:
\begin{equation}
\frac{(n - 1)\hat{\sigma}^2}{\chi^2_{\alpha/2}} - \frac{(n - 1)\hat{\sigma}^2}{\chi^2_{1 - \alpha/2}} \le 2\varepsilon.
\end{equation}

This yields:
\begin{equation}
n \ge 1 + \frac{\sigma^2}{\varepsilon} \cdot \frac{\Delta_\chi}{2}, 
\quad \text{where } \Delta_\chi = \frac{1}{\chi^2_{1 - \alpha/2}} - \frac{1}{\chi^2_{\alpha/2}}.
\end{equation}
It is worth noting that if the underlying data distribution deviates from Gaussian, a substantially larger sample size may be required to reliably estimate the variance—and hence stabilize $\gamma$.

\textbf{Proof details of \cref{prop:lp_bounded}.}
We want to highlight the credit of \cite{kumar2022fine} as most of the proof follows it, but several details are different. 

Since $Z^{T}X^{\top}  X {Z^{T}}^{\top}$ is invertible for most cases, there is a unique global minimum over a predictor's parameters $v$ to the loss optimized by linear-probing:
\begin{equation}
    \arg \min _v\left\|X {Z^{T}}^{\top} v-X {Z^{S}}^{\top} v^{S}\right\|_2^2=\left({Z^{T}} X^{\top} X {Z^{T}}^{\top}\right)^{-1} {Z^{T}} X^{\top} X {Z^{S}}^{\top} v^{S}.
\end{equation}
Assuming $\mathcal{Z}$ is a Reproducing Kernel Hilbert Space (RKHS),
the loss function thus is strongly convex in $v$ since the Hessian ${Z^{T}} X^{\top} X {Z^{T}}^{\top}$ is invertible. Then, the minima are unique for the gradient flow convergence in the RKHS for the parameters $v_{\mathrm{lp}}$ of $\mathcal{C}^{T'}$, where:
\begin{equation}
v_{\mathrm{lp}}^{\infty}=\left({Z^{T}} X^{\top} X {Z^{T}} ^{\top}\right)^{-1} {Z^{T}}  X^{\top} X{Z^{S}} ^{\top} v^{S},
\end{equation}
where the $v^{S}$ is the parameters of the unavailable $\mathcal{C}^{S}$.


We thus have the following definition:
\begin{equation}
    \begin{split}
\sqrt{{L\left(v_{\mathrm{lp}^{T'}}, Z^{T}\right)}} & =\left\|{Z^{S}}^{\top} v^{S}-{B^{T}}^{\top} v_{\mathrm{lp}}^{\infty}\right\|_2 \\
& \leq\left\|\left({Z^{S}}^{\top} v^{S}-{Z^{T}}^{\top} v^{S}\right)+\left({Z^{T}}^{\top} v^{S}-{Z^{T}}^{\top} v_{\mathrm{lp}}^{\infty}\right)\right\|_2 \\
& \leq \underbrace{\left\|{Z^{S}}^{\top} v^{S}-{Z^{T}}^{\top} v^{S}\right\|_2}_{(1)}+\underbrace{\left. \left\| {Z^{T}}^{\top} v^{S}-{Z^{T}}^{\top} v_{\mathrm{lp}}^{\infty}\right) \right\|_2}_{(2)}.
    \end{split}
\end{equation}

For term (1):
\begin{equation}
    \left\|B_{S}^{\top} v^{S}-{B^{T}}^{\top} v^{S}\right\|_2  \leq \sigma_{\max }\left(B_{S}-B^{T}\right)\left\|v^{S}\right\|_2 
     \leq \epsilon\left\|v^{S}\right\|_2 
     =\epsilon\left\|w_{\star}\right\|_2.
\end{equation}

Where we note that $\left\|v^{S}\right\|_2=\left\|w_{\star}\right\|_2$ and $w_{\star}=B_{S}^{\top} v^{S}$ where the rows of $B_{S}$ (columns of $B_{S}^{\top}$ ) are orthonormal.

Let $\Sigma=X^{\top} X$. For term (2), we first subtitute $v_{1 \mathrm{p}}^{\infty}$ and do some algebra (again noting that $\left\|v^{S}\right\|_2=\left\|w_{\star}\right\|_2$ ) to get:
\begin{equation}
\begin{split}
\left\|{B^{T}}^{\top} v^{S}-{B^{T}}^{\top} {v_{\mathrm{lp}}}^{\infty}\right\|_2 
& =\left\|{B^{T}}^{\top}\left({B^{T}} \Sigma {B^{T}}^{\top}\right)^{-1} {B^{T}} \Sigma\left({B^{T}}-B_{S}\right)^{\top} v^{S}\right\|_2 \\
& \leq \sigma_{\max }\left({B^{T}}^{\top}\left({B^{T}} \Sigma {B^{T}}^{\top}\right)^{-1} {B^{T}} \Sigma\right) \sigma_{\max }\left({B^{T}}-B_{S}\right)\left\|w_{\star}\right\|_2 \\
& \leq \frac{\sigma_{\max }\left({B^{T}}\right)^2 \sigma_{\max }(X)^2}{\sigma_{\min }\left(X {B^{T}}^{\top}\right)^2} \sigma_{\max }\left({B^{T}}-B_{S}\right)\left\|w_{\star}\right\|_2 \\
& \leq \frac{\sigma_{\max }\left({B^{T}}\right)^2 \sigma_{\max }(Z)^2}{\sigma_{\min }(Z)^2\left(\cos \theta_{\max }(R, S)\right)^2} \sigma_{\max }\left({B^{T}}-B_{S}\right)\left\|w_{\star}\right\|_2.
\end{split}
\end{equation}

Since $B^{T}$ has orthonormal rows, $\sigma_{\max }\left(B^{T}\right)=1$.$ \frac{\sigma_{\max }(Z)^2}{\sigma_{\min }(Z)^2} >=1$.
So it suffices to bound the quantities in the RKHS.

\section{Toy example details}
\label{app:Toy example details}

\textbf{Toy example training details.} 
To maintain simplicity, we employ the model with two MLPs and one LayerNorm layer in between. Following the standard fine-tuning process, our toy experiments also conduct training on the source data, fine-tuning on the target training set, and then testing on the rest of the target data. Specifically, the pretrianing stage trains all layers of the model; the fine-tuning stage only tunes the LayerNorm layer. After finetuning, we conduct the $\gamma$ rescaling on the tuned model. All examples are conducted for classification, and accuracy (acc) is used as the evaluation metric. Note that the reported accuracy results are in the range of $[0,1]$. 
After fine-tuning, each model’s LayerNorm $\gamma$ values are rescaled by $\lambda \in [0, 2]$, and the optimal value of $\lambda$ that leads to the best performance is chosen.
All toy examples run on the CPU. The default training seed is set to $42$.

\textbf{Toy data generation}
To simulate the source domain $X^{S}$, target training set $X^{T}$, and testing target sets $X^{T*}$ with class labels, we synthetically generate datasets based on Gaussian distributions with controlled variations in means and variances across multiple classes. Specifically, the source domain data are sampled from class-specific Gaussian distributions centered at fixed means with unit variance. Each distribution contributes an equal number of samples (i.e., $100$ samples per class), with labels assigned according to the corresponding Gaussian component.
To construct the target domain, we introduce domain shift by perturbing both the means and variances of the source distributions. Mean shifts are applied using predefined displacement vectors scaled by a \textit{mean shift scale} parameter, while variances are modified by scaling the original values using class-specific variance multipliers and a \textit{variance shift scale} factor. As a result, the target distributions exhibit both translational and scale shifts relative to their source counterparts, enabling controlled evaluation of model robustness under varying degrees of domain shift.
Similar to the source domain, $100$ labeled samples per class are generated for the target domain. A portion of these labeled target samples is used for training, while the remainder is reserved for testing.

We consider three settings for the number of classes: $2$, $4$, and $8$. In each setting, source samples are drawn from fixed Gaussian distributions, while target samples are drawn from the corresponding shifted distributions. The mean shift scale and variance shift scale parameters are varied independently within the range $[0, 2]$ in $11$ uniformly spaced steps to simulate increasing degrees of domain shift. For each configuration of the class count, shift parameters, and training conditions, the proportion of labeled target training data is varied among $[0.01, 0.05, 0.1, 0.3, 0.5]$, corresponding to $1, 5, 10, 30,$ and $50$ labeled samples per class to ensure that we cover one-shot, few-shot and normal fine-tuning scenarios. 
For each model, the shifts of $\gamma$ of LayerNorm are multiplied with the $\lambda$ that varies from $[0, 2]$ with $21$ uniformly spaced steps. 

\textbf{More results.}
The results of toy examples of different numbers of classes mentioned in the main paper are exhibited in \cref{fig:toy_res}.
Moreover, as shown in \cref{fig:toy_res_different_training_frac_vs_FSR}, the Fine-tuning Shift Ratio (FSR) decreases as the fraction of target samples used for training increases. In parallel, \cref{fig:toy_res_different_training_frac_vs_lambda} illustrates that the optimal value of $\lambda$ increases with larger training fractions. Notably, when the available target data is extremely limited (e.g., a fraction of 0.01, corresponding to one sample per class), the learned LayerNorm parameters $\gamma$ and $\beta$ become heavily biased by individual samples, causing the relationship between $\lambda$ and $FSR$ to become less reliable. Detailed visualization can be seen in \cref{fig:toy_res_different_training_frac}.


\begin{figure}[t]
    \centering
    \includegraphics[width=0.6\linewidth]{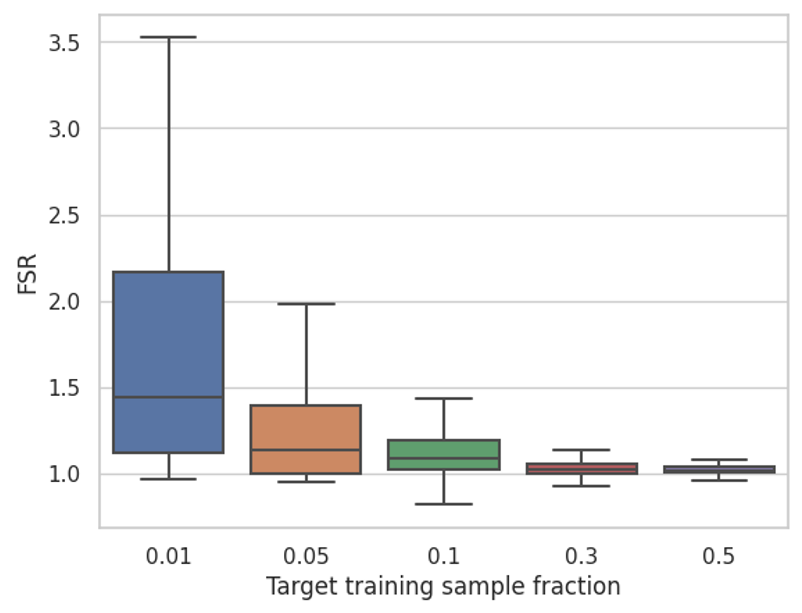}
    \caption{Toy example results: Fine-tuning Shift Ratio ($FSR$) against fractions of target samples used for training.
    }
    \label{fig:toy_res_different_training_frac_vs_FSR}
    \centering
    \includegraphics[width=0.6\linewidth]{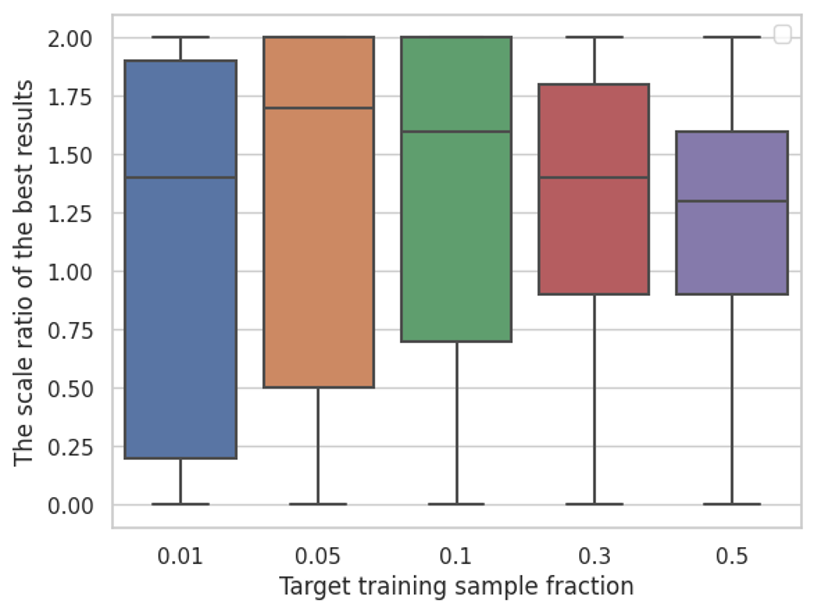}
    \caption{Toy example results: The scaling factor $\lambda$ against fractions of target samples used for training.
    }
    \label{fig:toy_res_different_training_frac_vs_lambda}
\end{figure}

\begin{figure}[t]
    \centering
    \includegraphics[width=0.7\linewidth]{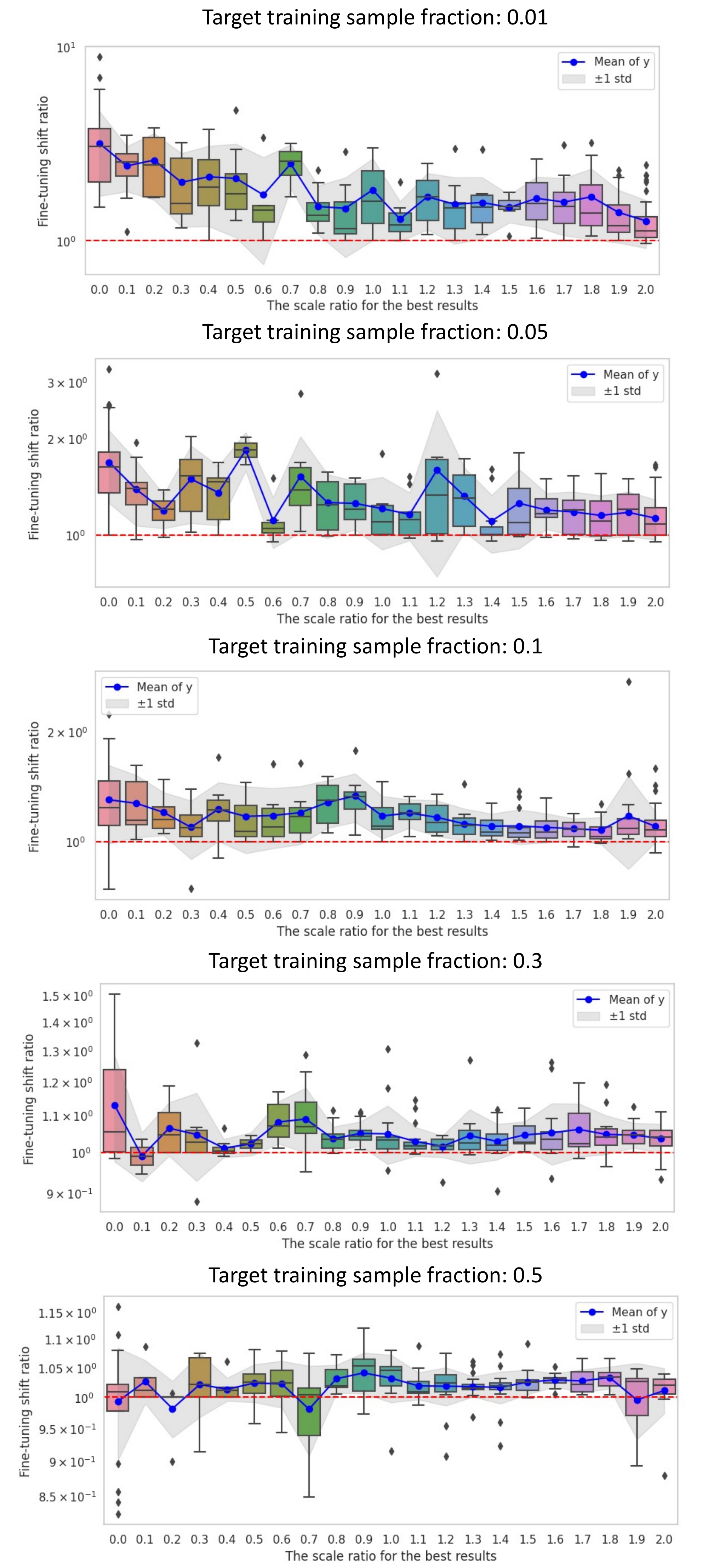}
    \caption{Toy example results: Statistical visualization of the relationship between the Fine-tuning Shift Ratio ($FSR$) and the scaling factor $\lambda$ that achieves the best performance, evaluated across varying fractions of target samples used for training.
    }
    \label{fig:toy_res_different_training_frac}
\end{figure}

\FloatBarrier
\section{More details about the proposed method}
\label{app:More details about the proposed method}

\subsection{Our method: More details and tricks}
To further enhance LayerNorm fine-tuning, we introduce two additional techniques: (i) increasing the feature dimension before prediction and (ii) applying a lightweight feature augmentation prior to attention pooling.
Expanding the feature dimension leads to a latent space where features become more separable, facilitating a more stable optimization process.
Additionally, we introduce a learnable scaling factor for the features before attention pooling, providing mild feature augmentation to promote better generalization.

\subsection{More results}
\textbf{Ablation study of proposed tricks.}
The results in \cref{fig:ablation_method_dataseed_4_conch_bach} illustrate the effects of each proposed trick.
Here, C1 denotes the baseline without feature dimension expansion, and C2 denotes doubling the feature dimension.
As the results indicate, our training pipeline already significantly boosts performance, and these two additional tricks provide further slight improvements.

\begin{figure}[ht]
    \centering
  \includegraphics[width=0.45\linewidth]{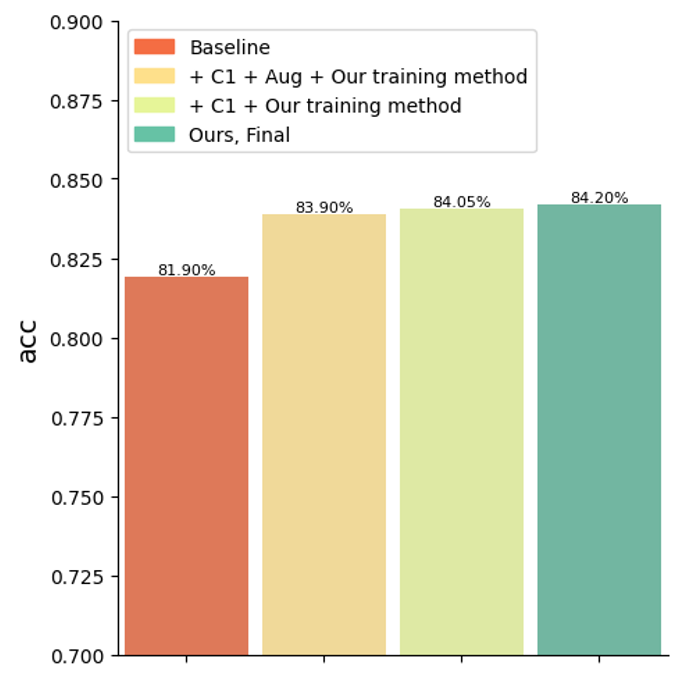}
    \caption{Ablation of proposed tricks of CONCH tuned on Bach. The results are averaged across various target training sample fractions. 
    }
  \label{fig:ablation_method_dataseed_4_conch_bach}
\end{figure}

\textbf{More comparisons of different LayerNorm fine-tuning approaches.} We offer more comparisons of different LayerNorm fine-tuning approaches in \cref{fig:path_ablation_titan}. It can be seen that our method yields the best or second-best results under most scenarios, bringing on average the best results among all competitors. 

\begin{figure}[ht]
    \centering
  \includegraphics[width=\linewidth]{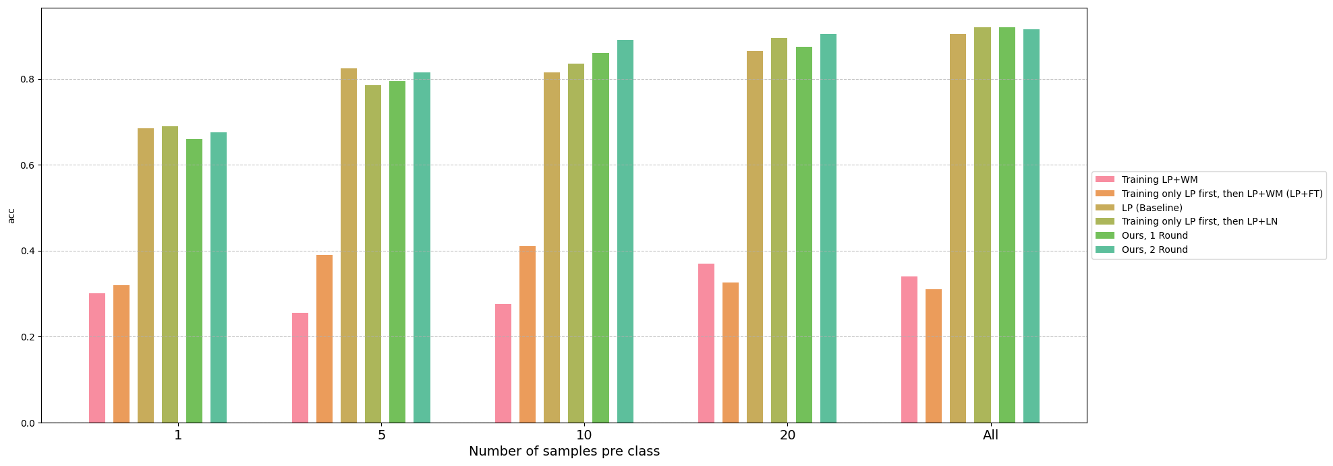}
    \caption{Our proposed method compared with various LayerNorm fine-tuning approaches. The used ViTF is TITAN, and the used dataset is Bach.   
    }
  \label{fig:path_ablation_titan}
\end{figure}

\FloatBarrier
\section{ID LayerNorm fine-tuning vs OOD LayerNorm fine-tuning: Experimental details and more results}
\label{app:Natural image experimental details}

\textbf{Experimental settings.}
For MAE, we pretrain the model by ourselves and then fine-tune it.
For the ID setting, the target domain is identical to the source domain, and fine-tuning is performed using labeled data from the same domain. In contrast, for the OOD setting, the target domains consist of the remaining domains not used during pretraining.  
To systematically analyze performance under varying data regimes, we fine-tune using randomly sampled $1\%$, $5\%$, $10\%$, and $50\%$ fractions of the labeled data from the target domain, while reserving the remaining data for evaluation. 

For large-scale pretrained CLIP and DINOv2, as it is difficult to determine whether it is ID or OOD settings, we tune them by using the subsets of overall domains, but test them across all domains. 
Here, we further test the large-scale pertaining ViTFs, OpenCLIP~\cite{radford2021learning} and DINOv2~\cite{oquab2023dinov2}, with LayerNorm fine-tuning on the domains of the DomainNet dataset used by  MAE, SUN~\cite{xiao2010sun}, DTD~\cite{cimpoi2014describing}, and Bach~\cite{aresta2019bach}. 
Like DomainNet for MAE, we also use randomly sampled $2\%$, $5\%$, $10\%$, and $100\%$ fractions of the labeled data from the target domain for both datasets and the official testing set for evaluation. Specifically, we use the official testing set for the performance evaluation.
$\lambda$ is searched in the range of $[0,2]$ and the best  $\lambda$ is chosen by the evaluation on the testing set. 
All experiments use one A100 GPU with 80GB of memory. 
The default training seed is set to $42$.

\subsection{Experimental details for DomainNet dataset}

\textbf{Pretraining MAE.} 
Following~\cite{he2022masked}, we adopt the self-supervised ViT-B/16 model as the backbone network.
The pretrianing stage uses the learning rate of $1\times10^{-4}$ for $500$ epochs. The mask ratio is set as $0.80$. {We adopted the vit\_base\_patch16\_224\_in21k as the initialization of the MAE.}

\textbf{Fine-tuning MAE.}
For fine-tuning, all approaches use the learning rate fixed at $5\times10^{-4}$, and all models are trained for a total of $100$ epochs. Our method performs each tuning round $20$ epochs and then switches to another round. The default batch size is $128$, but for extremely low data fractions (e.g., $0.01$), it is reduced to $12$ to account for the scarcity of training samples.

\textbf{Fine-tuning CLIP and DINOv2.}
For all experiments, the learning rate is fixed at $1\times10^{-3}$, and training is conducted for a total of $20$ epochs. The batch sizes are kept the same as those used in MAE fine-tuning. 
For our proposed method, the switch epoch is set to $2$.

\subsection{Experimental details for the SUN dataset}

\textbf{Fine-tuning CLIP and DINOv2.}
For all experiments, the learning rate is fixed at $5\times10^{-4}$ and weight decay is $1\times10^{-4}$, and training is conducted for a total of $20$ epochs. The batch sizes are $64$ for all experiments. Especially for our method on CLIP and DINOv2 fine-tuning, please refer to \cref{tab:CLIP_SUN_exp,tab:DINOv2_SUN_exp}.

\begin{table}[t]
    \caption{Experimental details for the SUN dataset: Our method for CLIP}
    \label{tab:CLIP_SUN_exp}
    \centering
    \begin{tabular}{c|cccc}
    \toprule
        Training fraction & Learning rate & Epochs & Switch epochs & Weight decay \\
        All& $5e-4$ &  10 & 5 & $1e-3$ \\
        $10\%$ & $5e-4$&  10 & 5 &$1e-5$ \\
        $5\%$ & $1e-3$&  10 & 2 &$1e-3$ \\
        $2\%$ & $3e-3$&  10 & 2 &$5e-3$ \\
    \bottomrule
    \end{tabular}
\end{table}

\begin{table}[t]
    \caption{Experimental details for the SUN dataset: Our method for DINOv2}
    \label{tab:DINOv2_SUN_exp}
    \centering
    \begin{tabular}{c|cccc}
    \toprule
        Training fraction & Learning rate & Epochs & Switch epochs & Weight decay \\
        All& $1e-4 $ &  20 & 5 & $1e-4$ \\
        $10\%$ & $2e-4$&  20 & 5 &$1e-4$ \\
        $5\%$ & $2e-4$&  20 & 5 &$1e-4$ \\
        $2\%$ & $2e-3$&  10 & 2 &$5e-3$ \\
    \bottomrule
    \end{tabular}
\end{table}

\subsection{Experimental details for the DTD dataset}

\textbf{Fine-tuning CLIP and DINOv2.}
For all experiments, the learning rate is fixed at $5\times10^{-4}$ and weight decay is $1\times10^{-4}$, and training is conducted for a total of $20$ epochs. The batch sizes are $64$ for all experiments. Especially for our method on CLIP and DINOv2 fine-tuning, please refer to \cref{tab:CLIP_DTD_exp,tab:DINOv2_DTD_exp}.

\begin{table}[t]
    \caption{Experimental details for the DTD dataset: Our method for CLIP}
    \label{tab:CLIP_DTD_exp}
    \centering
    \begin{tabular}{c|cccc}
    \toprule
        Training fraction & Learning rate & Epochs & Switch epochs & Weight decay \\
        All& $5e-4$ &  10 & 5 & $1e-3$ \\
        $10\%$ & $5e-4$&  10 & 5 &$1e-5$ \\
        $5\%$ & $1e-3$&  10 & 2 &$1e-3$ \\
        $2\%$ & $3e-3$&  10 & 2 &$5e-3$ \\
    \bottomrule
    \end{tabular}
\end{table}

\begin{table}[t]
    \caption{Experimental details for the SUN dataset: Our method for DINOv2}
    \label{tab:DINOv2_DTD_exp}
    \centering
    \begin{tabular}{c|cccc}
    \toprule
        Training fraction & Learning rate & Epochs & Switch epochs & Weight decay \\
        All& $1e-4 $ &  20 & 5 & $1e-6$ \\
        $10\%$ & $1e-3$&  20 & 5 &$1e-4$ \\
        $5\%$ & $1e-3$&  20 & 5 &$1e-4$ \\
        $2\%$ & $1e-3$&  10 & 2 &$5e-3$ \\
    \bottomrule
    \end{tabular}
\end{table}

\subsection{Experimental details for the Bach dataset}
\textbf{Fine-tuning CLIP and DINOv2.}
The experimental details are the same as the pathological ViTFs.

\subsection{More results}

\textbf{Applying $\lambda$ to $\gamma$ is more stable than $\beta$.} Please refer to the \cref{fig:natural_image_bias_weight}.

\begin{figure}[h]
    \centering
    \includegraphics[width=\linewidth]{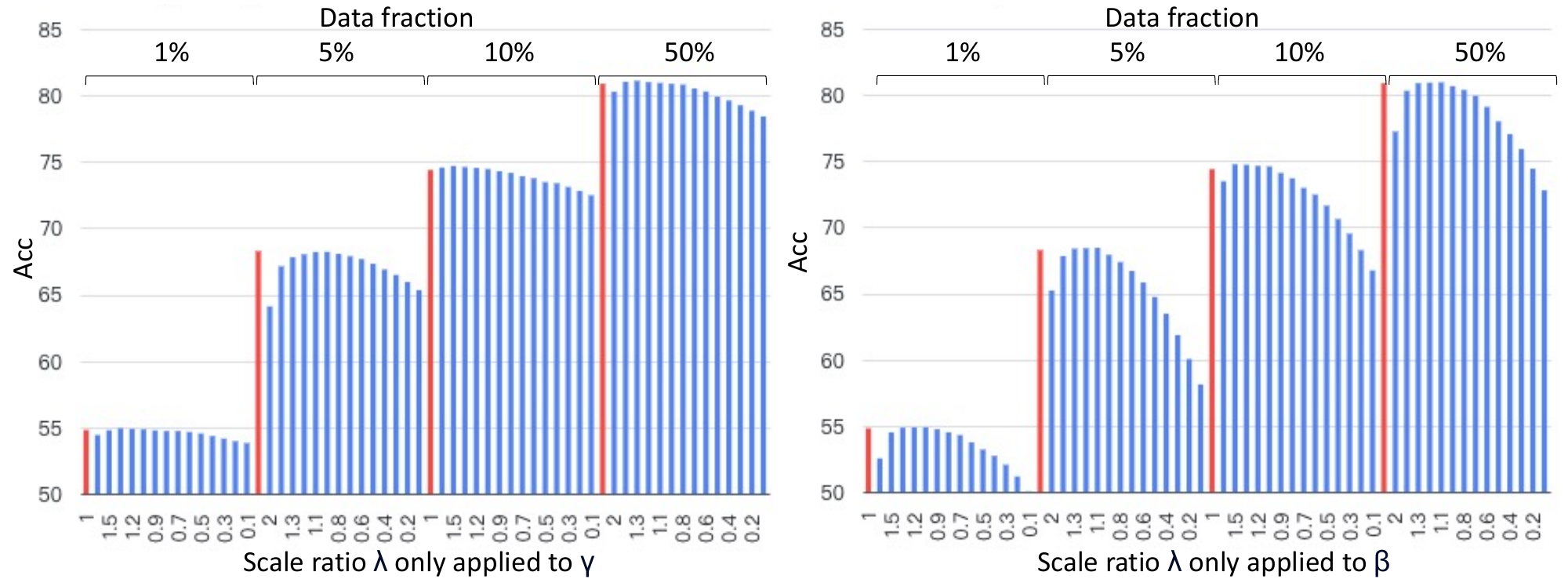}
    \caption{Comparison of applying $\lambda$ to $\gamma$ or $\beta$. The model adopted here corresponds to the ID:A to A experiments from \cref{tab:ID_OOD_res}.
    Though the tendencies are similar, it can be seen that the $\beta$ is way more sensitive to the value of $\lambda$ than $\gamma$. This reflects the analysis in \cref{sec:Rescaling LayerNorm shifts leads to improvements}. Therefore, we only apply the $\lambda$ to $\gamma$.
    }
    \label{fig:natural_image_bias_weight}
\end{figure}

\FloatBarrier
\section{LayerNorm fine-tuning for Pathology images ViTFs: Experimental details and more results}
\label{app:LayerNorm fine-tuning for Pathology images ViTFs: Experimental details and more results}

\subsection{Experimental details}
For a fair comparison, we train all ViTFs using the seed $1$ as the default. All experiments across all datasets and ViTFs are using $0.001$ as the learning rate with $100$ epochs. Especially for our method, we set the epoch of each round as $20$, i.e., switching the training stage at every $20$ rounds. All experiments use one A100 GPU with 80GB of memory. 
The default training seed is set to $42$.

\subsection{More results}

More results of pathology ViTFs can be found in \cref{fig:path_res_details,fig:path_res_details_each_fraction}.
More visualizations of the LayerNorm shifts can be seen in \cref{fig:path_norm_vis_detailed}.

We provide boxplots of our proposed method across various seeds in \cref{fig:path_different_seeds}. It can be seen that our method yields stable improvements with relatively smaller variances across all settings. This indicates the robustness of our proposed method.

\begin{figure}[ht]
    \centering
  \includegraphics[width=\linewidth]{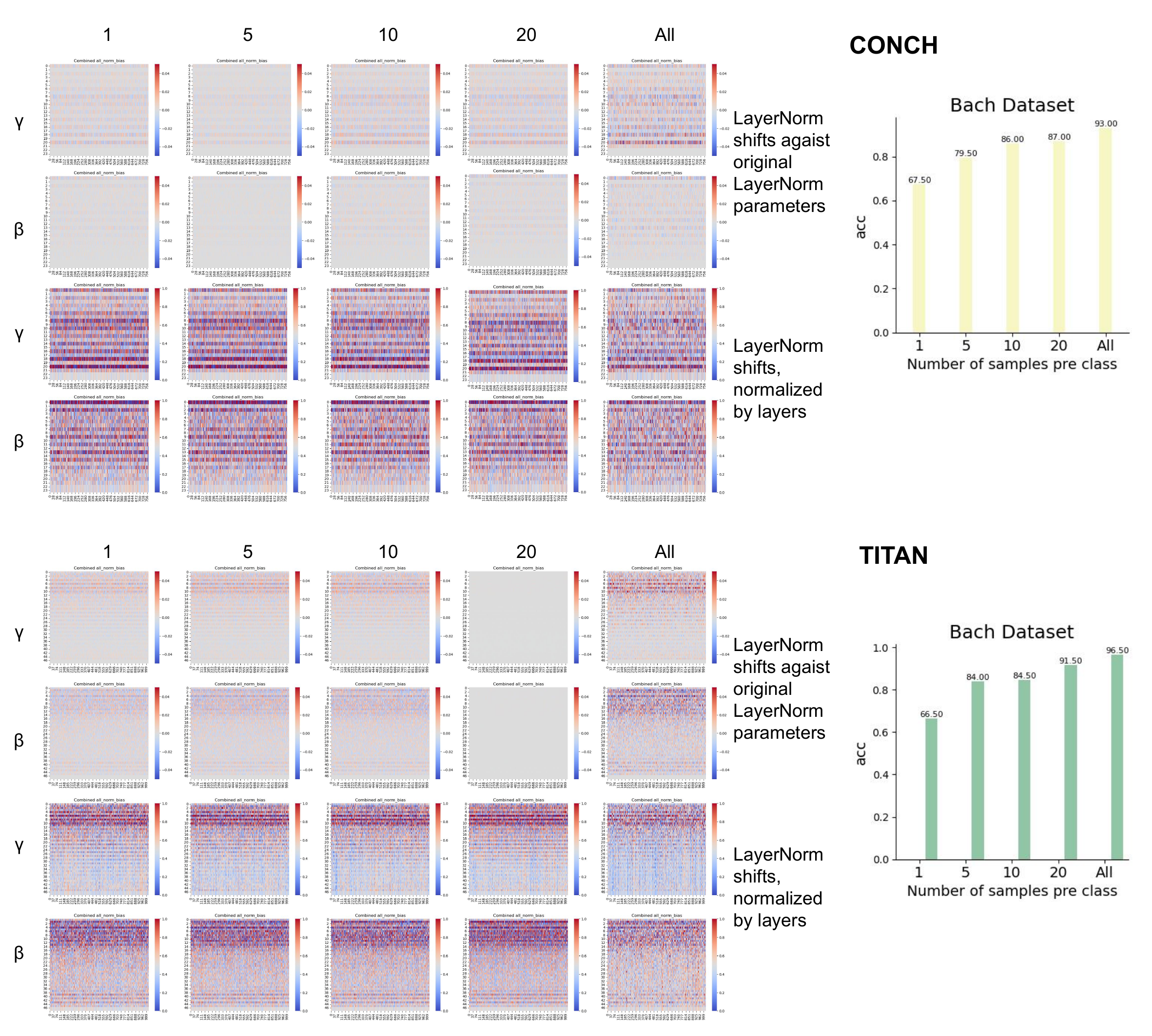}
    \caption{Visualizations of LayerNorm shifts on the Bach dataset for each setting. 
    }
  \label{fig:path_norm_vis_detailed}
\end{figure}

\begin{figure}[ht]
    \centering
  \includegraphics[width=\linewidth]{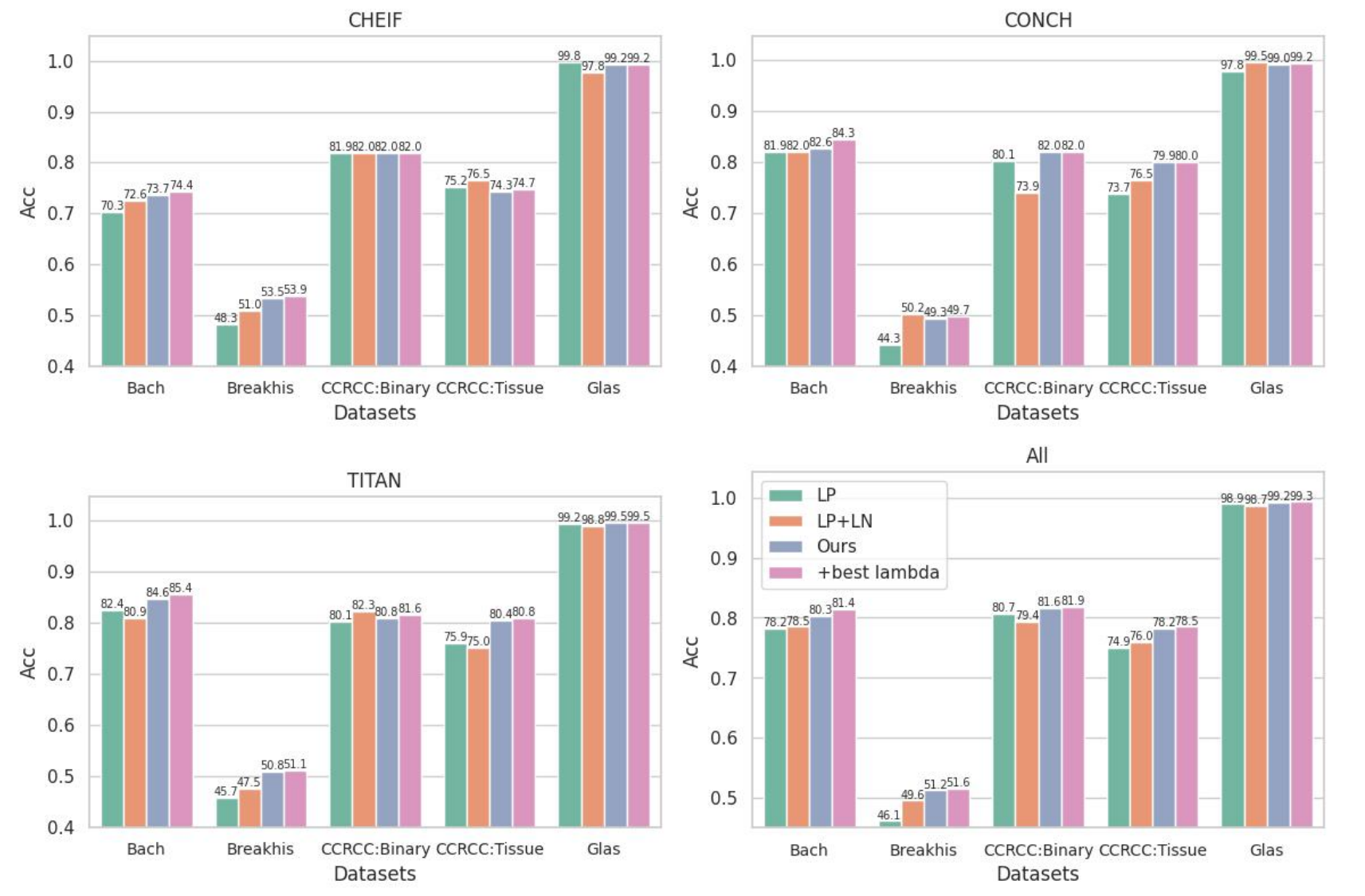}
    \caption{Results of fine-tuning pathological foundation models across datasets. Note here that all results are averaged across different numbers of target training samples per class. 
    }
  \label{fig:path_res_details}
\end{figure}

\begin{figure}[ht]
    \centering
    \includegraphics[width=\linewidth]{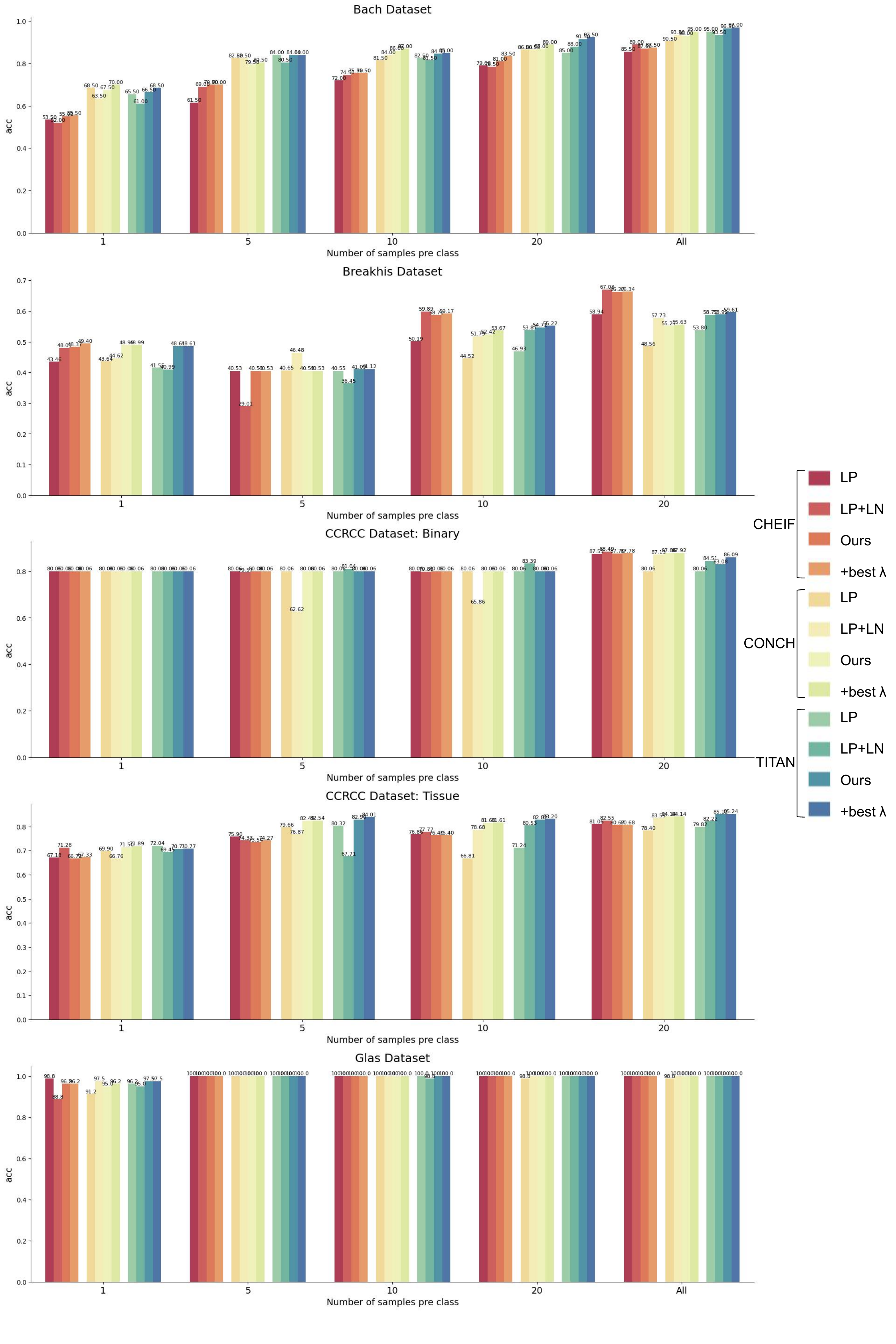}
    \caption{Results of each pathological dataset against different numbers of target training data for each class.
    }
    \label{fig:path_res_details_each_fraction}
\end{figure}

\begin{figure}[h]
    \centering
  \includegraphics[width=\linewidth]{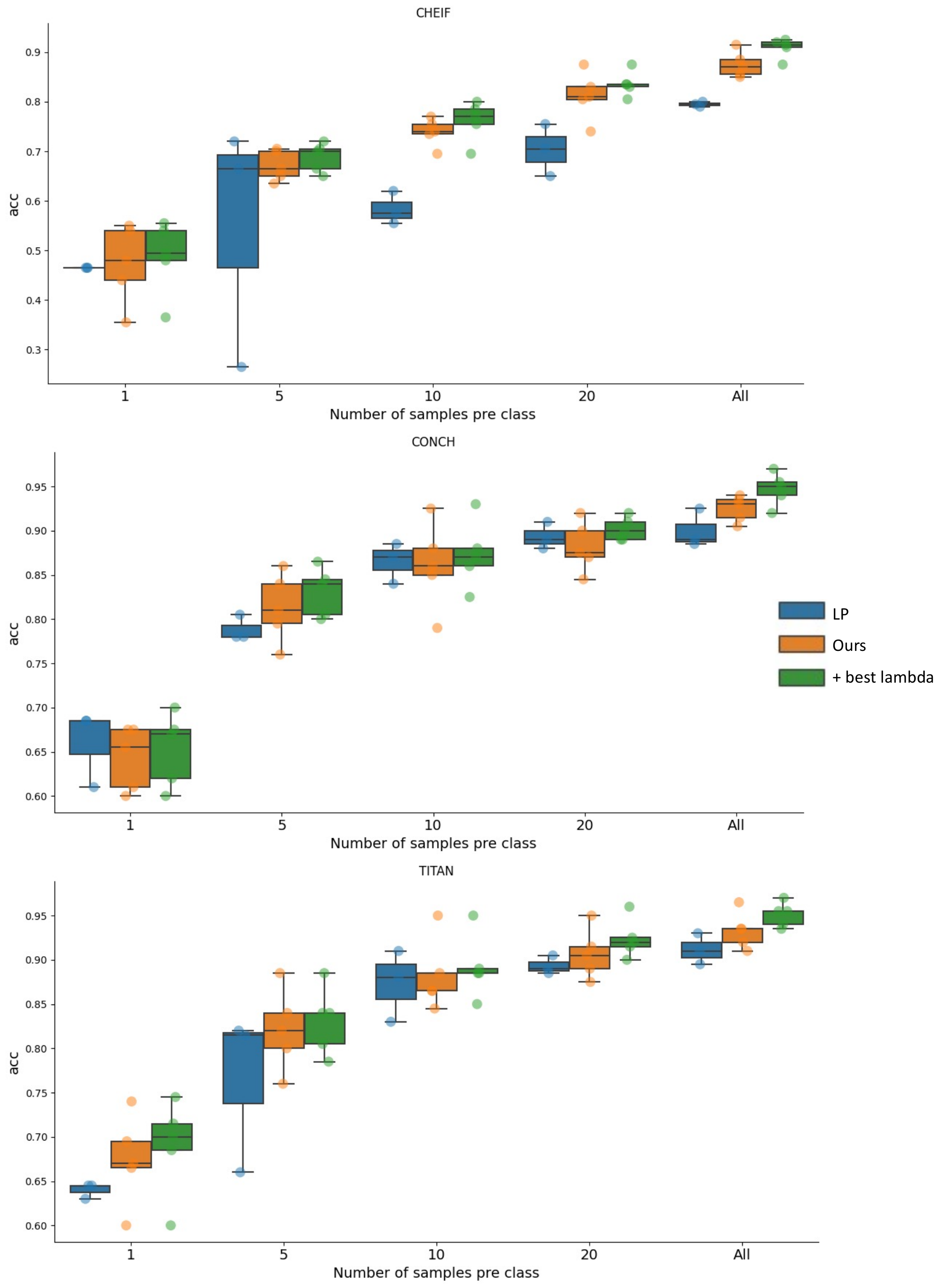}
    \caption{Our proposed method across various seeds for data splits in comparison to LP. The seeds used here are $[1,2,3,4,42]$.
    }
  \label{fig:path_different_seeds}
\end{figure}

\FloatBarrier
\subsection{Attempts to sparsify the LayerNorm shifts under limited target training samples}
\label{app:Attempts for sparsing the LayerNorm shifts while there are fewer target training samples}


As observed in \cref{sec:path_exp}, increasing the number of tuning samples generally results in sparser LayerNorm updates and improved performance. 
We have attempted to achieve sparsity under limited target training samples to enhance performance further. 
However, due to the randomness in the training process and the selection of target samples, achieving an appropriate level of sparsity seems to be non-trivial. 
We present our attempts here, hoping that these results may benefit those interested in further exploring this direction.

\textbf{Baseline.} We use the LP results as the baseline.

\textbf{SVD Keep first for gamma only.} 
All shifts in $\gamma$ parameters are concatenated into a single matrix, with columns corresponding to feature dimensions and rows corresponding to different layers. Singular Value Decomposition (SVD) is then applied to this matrix, and the top-$k$ singular components are retained. The reported results correspond to the optimal choice of $k$ selected based on validation performance.

\textbf{SVD Keep last for both gamma and beta.} All shifts in $\gamma$ and $\beta$ parameters are concatenated into a single matrix correspondingly, with columns corresponding to feature dimensions and rows corresponding to different layers. Singular Value Decomposition (SVD) is then applied to both matrices, and the last-$k$ singular components are retained. The reported results correspond to the optimal choice of $k$ selected based on validation performance.

\textbf{SVD Keep middle for both gamma and beta.}
All shifts in $\gamma$ and $\beta$ parameters are concatenated into a single matrix correspondingly, with columns corresponding to feature dimensions and rows corresponding to different layers. Singular Value Decomposition (SVD) is then applied to both matrices, and the middle-$k$ singular components are retained. The reported results correspond to the optimal choice of $k$ selected based on validation performance.

\textbf{SVD Keep first for beta only.}
All shifts in $\beta$ parameters are concatenated into a single matrix, with columns corresponding to feature dimensions and rows corresponding to different layers. Singular Value Decomposition (SVD) is then applied to this matrix, and the top-$k$ singular components are retained. The reported results correspond to the optimal choice of $k$ selected based on validation performance.

\textbf{SVD Keep first for both gamma and beta.} 
All shifts in $\gamma$ and $\beta$ parameters are concatenated into a single matrix correspondingly, with columns corresponding to feature dimensions and rows corresponding to different layers. Singular Value Decomposition (SVD) is then applied to both matrices, and the first-$k$ singular components are retained. The reported results correspond to the optimal choice of $k$ selected based on validation performance.

\textbf{Random drop for gamma only.} Randomly drop shifts in $\gamma$ by ratios. The reported results correspond to the optimal choice of the drop ratios selected based on validation performance.

\textbf{Random drop for beta only.}  Randomly drop shifts in $\beta$ by ratios. The reported results correspond to the optimal choice of the drop ratios selected based on validation performance.

\textbf{Only scale beta.} Rescaling shifts in $\beta$. The reported results correspond to the optimal choice of the rescale ratios selected based on validation performance.

\textbf{Only scale gamma (Ours).} Rescaling shifts in $\gamma$. The reported results correspond to the optimal choice of the rescale ratios selected based on validation performance.

\begin{figure}[ht]
    \centering
  \includegraphics[width=\linewidth]{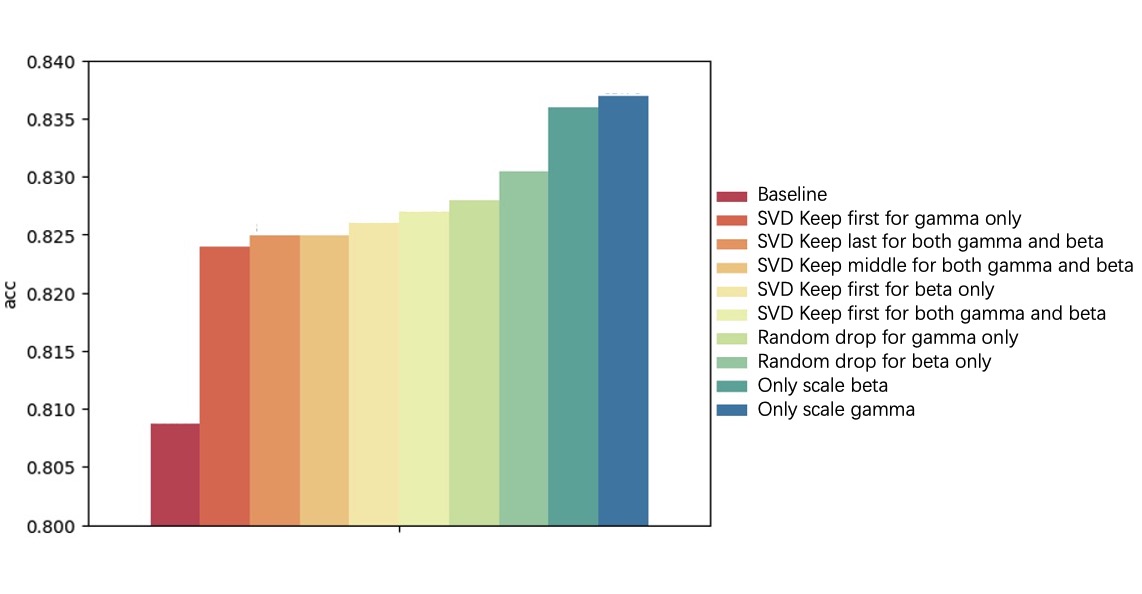}
    \caption{Different LayerNorm sparsity attempts.  The used dataset is Bach. The results are averaged across all factions and ViTFs. 
    }
  \label{fig:different_sparse}
\end{figure}

\textbf{Result.} As shown in \cref{fig:different_sparse}, Only scale gamma (Ours) achieves the best results. Moreover, we find that the SVD-based approaches often fail to bring improvements when there are more training samples (e.g., $20$ per class) and may not be able to be effective for all modes (e.g., CHEIF).  
Thus, while sparsifying LayerNorm updates in low-data regimes may offer potential benefits, it remains a non-trivial challenge due to the complex, model-specific nature of these adaptations and the inherent uncertainty associated with the selected target training samples.

\clearpage

\end{document}